\documentclass[10pt,twocolumn,letterpaper]{article}

\usepackage{iccv}
\usepackage{times}
\usepackage{epsfig}
\usepackage{graphicx}
\usepackage{amsmath}
\usepackage{amssymb}
\usepackage{amsthm}
\usepackage{booktabs}
\usepackage{color,algorithm}
\usepackage{algorithmic}
\usepackage{pifont}% http://ctan.org/pkg/pifont
\usepackage{makecell}
\usepackage{hyperref}
\hypersetup{
    colorlinks=true,
    filecolor=magenta,      
    urlcolor=magenta
    }

\newcommand{\cmark}{\ding{51}}%
\newcommand{\xmark}{\ding{55}}

\usepackage[accsupp]{axessibility}  % Improves PDF readability for those with disabilities.

\theoremstyle{plain}
\newtheorem{theorem}{Theorem}[section]

\newtheorem{lemma}[theorem]{Lemma}
\newtheorem{proposition}[theorem]{Proposition}

\theoremstyle{definition}
\newtheorem{definition}[theorem]{Definition}

\theoremstyle{remark}

\iccvfinalcopy % *** Uncomment this line for the final submission

\def\iccvPaperID{5983} % *** Enter the ICCV Paper ID here

\ificcvfinal\fi

\begin{document}

%%%%%%%%% TITLE
\title{Manifold Matching via Deep Metric Learning for Generative Modeling}

\author{Mengyu Dai\thanks{Equal contributions.}\\
Microsoft\\
%Institution1 address\\
{\tt\small mendai@microsoft.com}
% For a paper whose authors are all at the same institution,
% omit the following lines up until the closing ``}''.
% Additional authors and addresses can be added with ``\and'',
% just like the second author.
% To save space, use either the email address or home page, not both
\and
Haibin Hang\footnotemark[1]\\
University of Delaware\\
%First line of institution2 address\\
{\tt\small haibin@udel.edu}
}

\maketitle
% Remove page # from the first page of camera-ready.
\ificcvfinal\thispagestyle{empty}\fi

\begin{abstract}
We propose a manifold matching approach to generative models which includes a distribution generator (or data generator) and a metric generator. 
In our framework, we view the real data set as some manifold embedded in a high-dimensional Euclidean space.
The distribution generator aims at generating samples that follow some distribution condensed around the real data manifold. 
It is achieved by matching two sets of points using their geometric shape descriptors, such as centroid and $p$-diameter, with learned distance metric; 
the metric generator utilizes both real data and generated samples to learn a distance metric which is close to some intrinsic geodesic distance on the real data manifold. The produced distance metric is further used for manifold matching.
The two networks learn simultaneously during the training process.
We apply the approach on both unsupervised and supervised learning tasks: in unconditional image generation task, the proposed method obtains competitive results compared with existing generative models; in super-resolution task, we incorporate the framework in perception-based models and improve visual qualities by producing samples with more natural textures. Experiments and analysis demonstrate the feasibility and effectiveness of the proposed framework.
%Code is available at \href{https://github.com/dzld00/pytorch-manifold-matching}{https://github.com/dzld00/pytorch-manifold-matching}.
\end{abstract}

\section{Introduction}\label{sec:introduction}
Deep generative models including Variational Autoencoder (VAE) \cite{VAE}, Generative Adversarial Networks (GAN) \cite{goodfellow2014generative} and their variants \cite{wgan2017, li2017mmd, Lee_2019_CVPR, Rolinek_2019_CVPR, Ding_2020_CVPR, zhang2021image, Rao_2020_CVPR} have achieved great success in generative tasks such as image and video synthesis, super-resolution (SR), image-to-image translation, text generation, neural rendering, etc. 
The above approaches try to generate samples which mimic real data by minimizing various discrepancies between their corresponding statistical distributions, such as using KL divergence \cite{VAE}, Jensen-Shannon divergence \cite{goodfellow2014generative}, Wasserstein distance \cite{wgan2017}, Maximum Mean Discrepancy \cite{li2017mmd} and so on. 
These approaches focused on the data distribution aspect and did not pay enough attention to the underlying metrics of these distributions. 
The interplay between distribution measure and its underlying metric is a central topic in optimal transport (cf.~\cite{villani2008optimal}). 
Despite that researchers have successfully employed optimal transport theory in generative models \cite{wgan2017,Wu_2019_CVPR,Deshpande_2019_CVPR}, 
simply assuming the underlying metric to be Euclidean metric may neglect some rich information lying in the data  \cite{10.5555/645504.656414}.
In addition, %the intermediate relations between data points during training process cannot be easily revealed. 
although the above approaches are validated to be effective, successful training setups are mostly based on empirical observations and lack of physical interpretations. 
% As the example shown in Fig~\ref{fig:dist_sample}, while a traditional GAN discriminator can only distinguish real and fake samples, the proposed approach can utilize learned distance for better interpretations of individual data points.
% \begin{figure}[ht]
% 	\centering  
% 	\includegraphics[width=2.1in, height = 1.3in]{figs/dist_example.png} 
% 	\caption{An example of learned distances between real and generated samples during training using our method. Image 0 is a generated sample. Image 1 and 2 are ``000085'' and ``000054'' in CelebA dataset respectively. $d_{ij}$ represents scaled learned distance between image $i$ and $j$.} 
% 	\label{fig:dist_sample}
% 	 \vspace{-1em}
% \end{figure}

In this paper we bring up a geometric perspective which serves as an important parallel view of generative models as GANs. Table~\ref{tab:comp_GAN_MM} summarizes the main differences between classic GANs and our proposed (so-called MvM) framework. 
Instead of directly matching statistical discrepancies under Euclidean distance, we provide a more flexible framework which is built upon learning the intrinsic distances among data points.
Specifically, we treat the real data set as some manifold embedded in high-dimensional Euclidean space, and generate a fake distribution measure condensed around the real data manifold by optimizing a \emph{Manifold Matching (MM)} objective. The MM objective is built on shape descriptors, such as centroid and $p$-diameter with respect to some proper metric learnt by a metric generator using \emph{Metric Learning (ML)} approaches. During training process, the (fake data) distribution generator and the metric generator work interchangeably and produces better distribution (metric) that facilitates the efficient training of metric (distribution) generator.
The learned distances can not only be used to formulate energy-based loss functions \cite{LeCun06atutorial} for MM, but can also reveal meaningful geometric structures of real data manifold. %and describe relations between data points with meaningful interpretations.  
\begin{table}[ht]
    \scriptsize     %\small
	\centering
	\caption{Main differences between GANs and MvM.}
	\label{tab:comp_GAN_MM}
	\begin{tabular}{c|cc}
	    \hline
	    Differences & GANs & MvM \\
	    \hline
	    Main point of view & statistics & geometry 
	    \\
		\hline
		Matching terms & \makecell{means, moments, etc.} & \makecell{centroids, $p$-diameters}  \\
		\hline
		Matching criteria & \makecell{statistical discrepancy} & 
		\makecell{learned distances}
		\\
		\hline
		\makecell{Underlying metric} & \makecell{default Euclidean} & \makecell{learned intrinsic} 
		\\
		\hline
		\makecell{Objective functions} & \makecell{one min-max value function} & \makecell{two distinct objectives } 
		\\
		\hline
	\end{tabular}
    \vspace{-1em}
\end{table}

We apply the proposed framework on two tasks: unconditional image generation and single image super-resolution (SISR). 
We utilize unconditional image generation task as a validation of the feasibility of the approach; and further implement a supervised version of the framework on SISR task to demonstrate its advantage. 
%Dong \cite{7115171} et al. first proposed using convolution neural network to learn mappings between LR and HR pairs, after which many \cite{} focused on building network architectures and aimed at improving distortion-oriented performance. 
%To improve perceptual quality of generated images, Ledig et al.~\cite{ledig2017photo} proposed SRGAN which incorporated GAN in their framework, along with other GAN-based works \cite{sajjadi2017enhancenet,wang2018esrgan,Soh_2019_CVPR} developed afterwards.
%Perception-oriented models ~\cite{ledig2017photo,sajjadi2017enhancenet,wang2018esrgan,Soh_2019_CVPR} often generate images with artifacts which do not appear in HR ground-truth. 
% An example to illustrate this issue is mentioned at p-8 in~ \cite{sajjadi2017enhancenet}. As showed in Fig.~\ref{fig:BSD100_93},~\cite{sajjadi2017enhancenet} (b) generates vertical stripes on the object which do not exist in original image (a). The same issue also occurs in other perception-based methods such as in~\cite{Soh_2019_CVPR} (c). 
%which aims to produce more natural textures from images and mitigate the artifact issue lying in perception-based models.
%without using a heavy generator backbone. Meanwhile on the other hand it can also be incorporated into various frameworks to potentially improve results.
%As for unconditional image generation task, without using any labelled information, our method is able to generate visually appealing results compared to those generated by classic GAN frameworks \cite{wgan2017, li2017mmd}. 
Our main contributions are:
(1) We propose a manifold matching approach for generative modeling, which matches geometric descriptors between real and generated data sets using distances learned during training;
(2) We provide a flexible framework for modeling data and building objectives, where each generator has its own designated objective function; 
(3) We conduct experiments on unconditional image generation task and SISR task which validates the effectiveness of the proposed framework. 

\section{Related Work}
\noindent 
{\bf Manifold Matching:} 
Shen {\etal}~\cite{Shen2017ManifoldMU} proposed a nonlinear manifold matching algorithm using shortest-path distance and joint neighborhood selection, and illustrated its usage in medical imaging applications. %Validation was performed using brain network data.
Priebe {\etal} \cite{Priebe10manifoldmatching} investigated in manifold matching task from the perspective of
jointly optimizing the fidelity and commensurability, with an application in document matching. 
Lim and Ye \cite{lim2017geometric} decomposed GAN training into three geometric steps and used SVM separating hyperplane that has the maximal margins between classes.
Lei {\etal} \cite{lei2017geometric} showed the intrinsic relations between optimal transportation and convex geometry, and further used it to analyze generative models.
Genevay {\etal} \cite{pmlr-v84-genevay18a} introduced the Sinkhorn loss in generative models, based on regularized
optimal transport with an entropy penalty.
Shao {\etal} \cite{Shao_2018_CVPR_Workshops} introduced ways of exploring the Riemannian geometry of manifolds learned by 
generative models, and showed that the manifolds learned by deep generative models are close to zero curvature.
Park {\etal}~\cite{park2018mmgan} added a manifold matching loss in GAN objectives which tried to match distributions
%centroid and radius of data sets by 
using kernel tricks. However, the learning process highly relies on optimizing objectives in the original GAN framework. In addition, without proper metrics, using pre-defined kernels may fail to match the true shapes of data manifolds.
%in which case the approach is similar to using maximum-mean discrepancy as additional regularization in a GAN framework.        
In our work, the manifold matching is implemented using geometric descriptors under proper metrics learned by a metric generator.

\noindent 
{\bf Deep Metric Learning:} Among rich sources of literature on deep metric learning, we mainly focus on a few that are related to our work. Xing {\etal} \cite{xing2003distance} first proposed distance metric learning with applications to improve clustering performance. Hoffer and Ailon implemented deep metric learning with Triplet network \cite{hoffer2015deep} which aimed to learn useful representations through distance comparisons. 
Duan {\etal} \cite{Duan_2018_CVPR} proposed a deep adversarial metric learning framework to generate synthetic hard negatives from negative samples. The hard negative generator and feature embedding were trained simultaneously to learn more precise distance metrics. The metrics were learned in a supervised fashion and then used in classification tasks. Unlike \cite{Duan_2018_CVPR}, our approach utilizes geometric descriptors for matching data manifolds to generate data without using any labelled information.
Mohan {\etal} \cite{Mohan_2020_CVPR} proposed a direction regularization method which tried to improve the representation space being learnt by guiding the pairs move towards right directions in the metric space. 
In this work, we utilize the approach in \cite{Mohan_2020_CVPR} for metric learning implementation, 
%while one can use other techniques to approximate distance metrics as well. 

\noindent 
{\bf Perception-Based SISR:} SISR aims to recover a high-resolution (HR) image from a low-resolution (LR) one.
Ledig {\etal} \cite{ledig2017photo} first incorporated adversarial component in their objective and achieved high perceptual quality. However, SRGAN can generate observable artifacts such as undesirable noise and whitening effect. Similar issues were also mentioned in Sajjadi {\etal} \cite{sajjadi2017enhancenet}. Wang {\etal} \cite{wang2018esrgan} proposed ESRGAN which improved perceptual quality by improving SRGAN network architecture and combining PSNR-oriented network and a GAN-based network to balance perceptual quality and fidelity. 
Ma {\etal} \cite{Ma_2020_CVPR} proposed a gradient branch which provides additional structure priors for the SR process. Utilization of the gradient branch need corresponding network architectures to be equipped with. 
Soh {\etal} \cite{Soh_2019_CVPR} introduced natural manifold discriminator which tried to distinguish real and generated noisy and blurry samples. 
Since the natural manifold discriminator focuses on classifying certain types of manually generated fake data, one following question is: can we find a more robust way to learn useful information from real data?
Thus in this case we view one usage of our work in SR task as an extension of the natural manifold discriminator. 
%However, the domain of generated noisy and blurry images is limited and hard to cover real-world sophisticated cases. 
Some other recent methods \cite{Lim_2017_CVPR_Workshops,Cai_2019_ICCV,Qiu_2019_ICCV,Liu_2020_CVPR,Guo_2020_CVPR,Bhat_2021_CVPR} mainly work on improving network architectures which are not directly comparable to our approach.
In this paper we focus on objectives regardless of generator architectures, while the method can be incorporated into existing works. 
%to potentially improve results.

%A gradient loss is proposed which imposes a second-order restriction on the super-resolved images. The gradient branch incorporates several intermediate-level representations from the SR, which helps generative networks concentrate more on geometric structures. 
%In addition to \cite{Ma_2020_CVPR}, some other recent works also proposed matching features of intermediate layers in generator network or patches at different scales \cite{Guo_2020_CVPR,Mei_2020_CVPR,Liu_2020_CVPR}, which need corresponding generator network architectures to be equipped with. 

% \noindent 
% {\bf Unconditional Image Generation:} Unlike in super-resolution or conditional image generation tasks where labelled ground-truth is used to guide training process, we choose unconditional image generation task to further validate the effectiveness of the proposed framework. Existing techniques including GAN \cite{GAN} and its variants \cite{wgan2017,li2017mmd,Ansari_2020_CVPR} have been widely used in the task. While GAN tries to learn a data distribution that mimics the real data distribution, we try to formulate the problem from manifold matching point of view using distance-based matching criteria.
% In this work, we first propose manifold matching using deep adversarial metric learning for image generation tasks.

\section{Methodology}
\subsection{Proposed Framework}
We propose a metric measure framework for generative modeling which contains a distribution generator $f_\theta:\mathbb{R}^m\rightarrow\mathbb{R}^D$ and a metric generator $g_w:\mathbb{R}^D\rightarrow\mathbb{R}^n$.

$$\mathbb{R}^m\xrightarrow{\hspace*{0.3cm} f_\theta \hspace*{0.3cm}}\mathbb{R}^D\xrightarrow{\hspace*{0.3cm} g_w \hspace*{0.3cm}}\mathbb{R}^n$$
%$$(f_\theta)_\ast:\mathcal{P}(\mathbb{R}^m)\rightarrow\mathcal{P}(\mathbb{R}^D)$$
%$$(g_w)^\ast:\mathcal{D}(\mathbb{R}^n)\rightarrow\mathcal{D}(\mathbb{R}^D)$$
The metric generator $g_w$ would produce some metric $d$ on $\mathbb{R}^D$ to be the pullback of the Euclidean metric $d_E$ on $\mathbb{R}^n$ (see Definition~\ref{D:pullback_metric}).
The distribution generator $f_\theta$ would produce some measure $\mu$ on $\mathbb{R}^D$ to be the pushforward of some prior distribution $\nu$ %(e.g., the normal distribution $\mathcal{N}(0,I)$) 
on $\mathbb{R}^m$ (see Definition~\ref{D:pushforward_measure}). 
In implementations, $m, D, n$ represent dimensions of the input variable, target image, and image embedding respectively.
 
Now we have a metric measure space $(\mathbb{R}^D,d,\mu)$. The space of real data is viewed as some manifold $M\subseteq\mathbb{R}^D$ embedded in Euclidean space. 
The measure $\mu$ is said to be condensed around manifold $M$ if the majority of the measure $\mu$ is distributed nearby $M$ (see Fig.~\ref{F:condensed}).
The manifold $M$ is called totally geodesic (or ``straight'') with respect to metric $d$ if for any two points $a,b\in M$, the shortest path measured by metric $d$ stays on $M$ (see Fig.~\ref{F:straight}).
Using the generators $f_\theta$, $g_w$ modeled as neural networks, we aim to find proper parameters $\theta$ and $w$ such that the induced measure $\mu$ and metric $d$ satisfies:\\
(1) $\mu$ is as condensed as possible around $M$;\\
(2) $M$ is as ``straight'' as possible under $d$.

\begin{figure}[ht]
    \centering
    \begin{tabular}{c c}
     \includegraphics[scale=0.12]{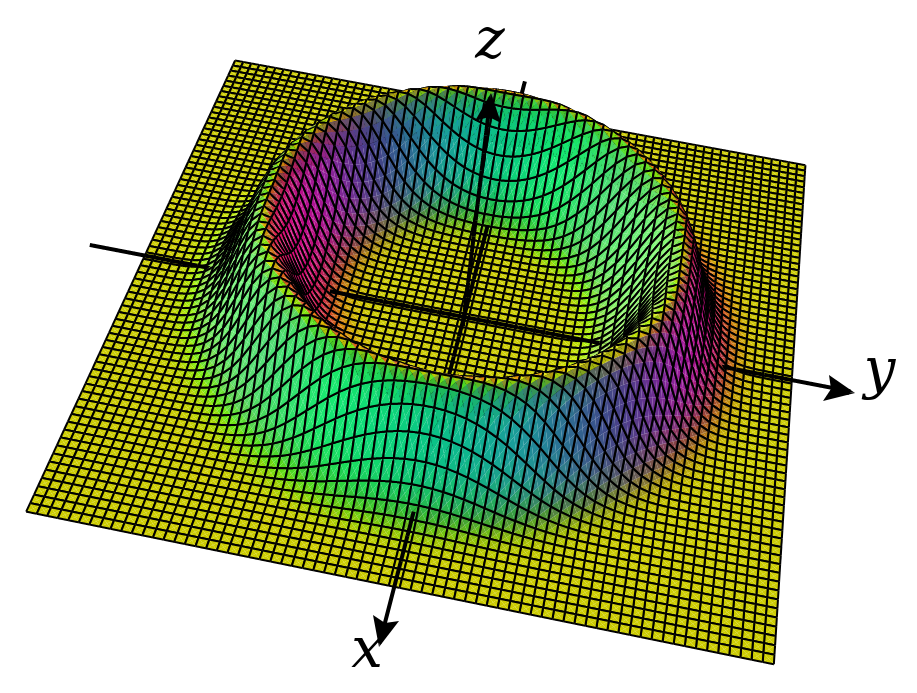}&
     \includegraphics[width=1.4in]{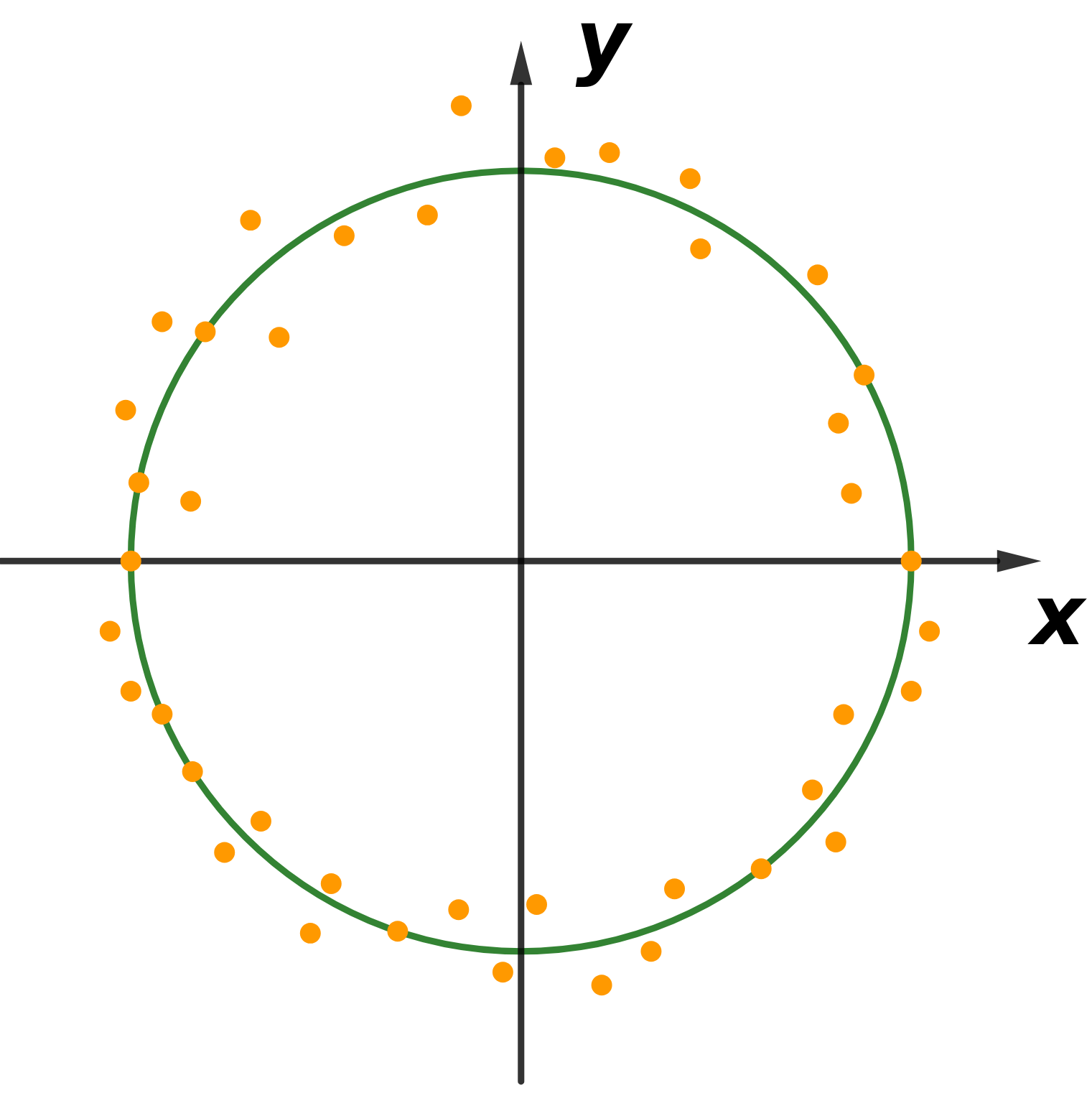}\\
     (i) & (ii)
    \end{tabular}
    \caption{(i) The probability density function of a distribution $\mu$ which condensed around a circle (manifold) $M\subseteq \mathbb{R}^2$; (ii) The orange dots represents random samples of $\mu$ and the green circle represents the real data manifold $M\subseteq\mathbb{R}^2$. }
    \label{F:condensed}
\end{figure}
\begin{figure}[ht]
    \centering
    \begin{tabular}{c c}
         \includegraphics[scale=0.105]{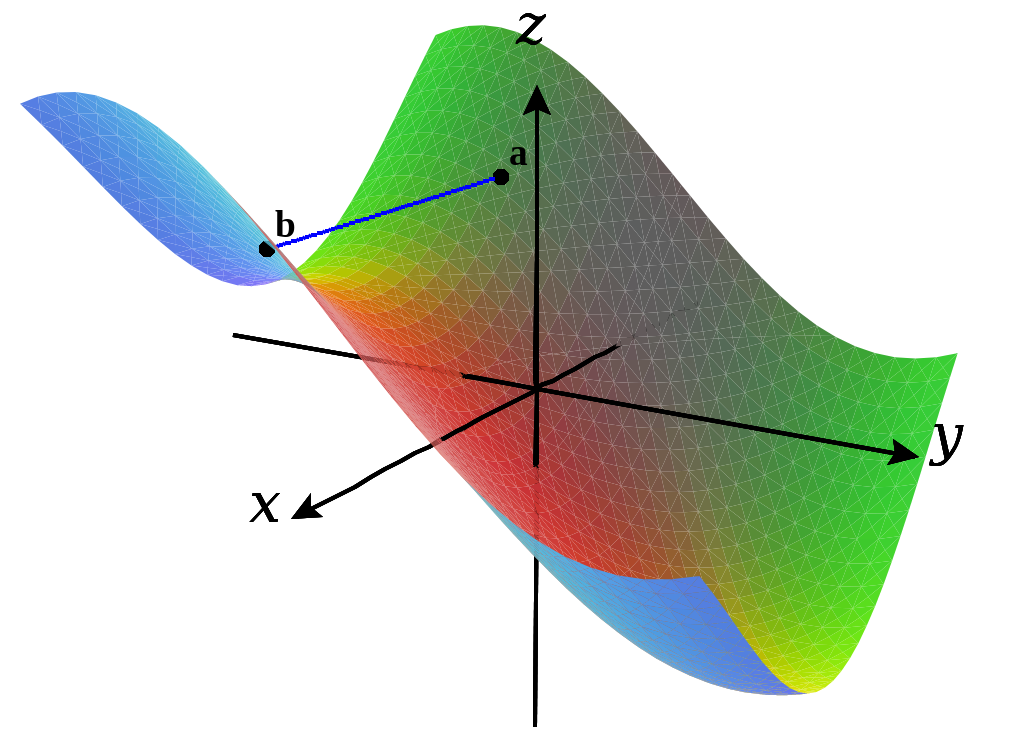}&\includegraphics[scale=0.105]{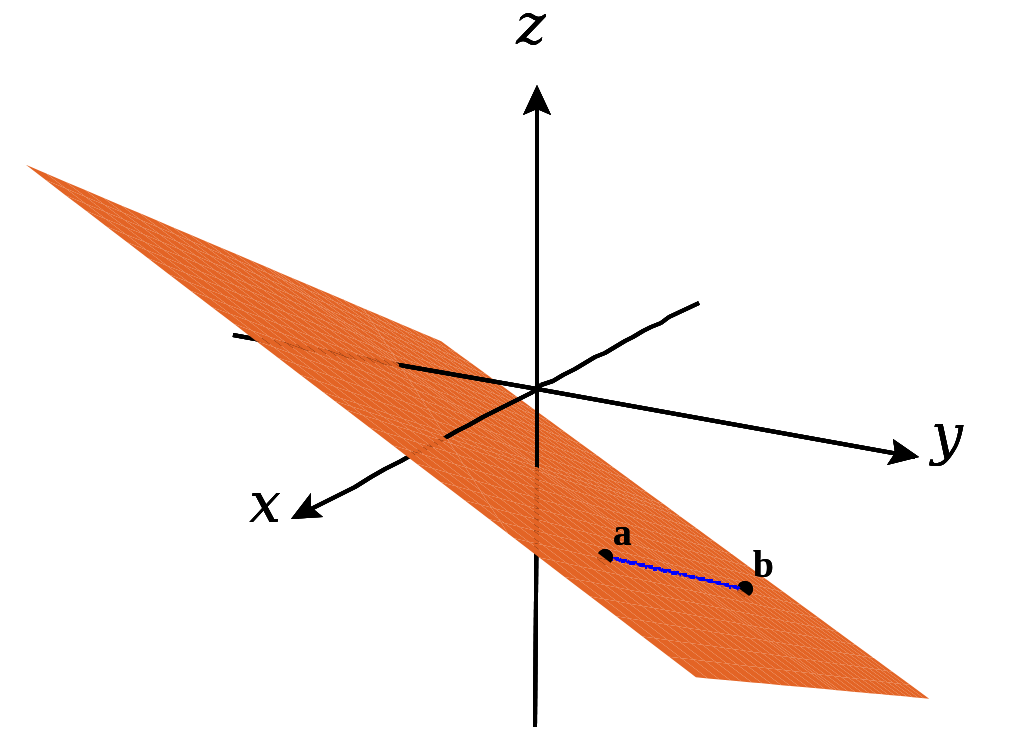}  \\
         (i) & (ii) 
    \end{tabular}
    \caption{The blue segment represents the shortest path between two points $a,b$ with respect to Euclidean distance $d_E$. (i) A non-geodesic sub-manifold of $\mathbb{R}^3$ under $d_E$; (ii) A geodesic sub-manifold of $\mathbb{R}^3$ under $d_E$.}
    \label{F:straight}
\end{figure}
Generally speaking, the two networks would produce a sequence of metrics $\{d_{2i}\}_{i\geq 0}$ and a sequence of measures $\{\mu_{2i+1}\}_{i\geq 0}$ inductively and alternatively as follows: (i) Let $d_0 = d_E$; Then for any $i>0$, 
(ii) derive measure $\mu_{2i-1}$ using manifold matching based on metric $d_{2i-2}$; 
(iii) derive metric $d_{2i}$ using metric learning based on measure $\mu_{2i-1}$.
$$d_0=d_E \rightsquigarrow \mu_1 \rightsquigarrow d_2 \rightsquigarrow \mu_3 \rightsquigarrow d_4 \rightsquigarrow \cdots$$      

%The reason we call this framework a coorperative generative network is that, in each step, one of the generator take advantage of the updated information provided by the other generator. %The basic assumption we use in our paper which lacks solid theoretical proof is that 
%\begin{assumption}
%The cooperative generative network
%\end{assumption}
%and reach convergence together.
In the following we introduce how to implement manifold matching and metric learning in details.

\subsection{Manifold Matching}
Let $\mathcal{P}(X)$ be the set of all probability measures on space $X$. Let $\mathcal{D}(X)$ be the set of all metrics on space $X$. %(Proofs for all properties appeared in this section can be found in appendix.)
\begin{definition}[Pushforward measure]\label{D:pushforward_measure}
Given a map $f:X\rightarrow Y$ and a probability measure $\mu\in \mathcal{P}(X)$, the push forward measure $f_\ast(\mu)\in\mathcal{P}(Y)$ is defined as: for any measurable set $A\subseteq Y$,
$$(f_\ast \mu)(A):=\mu(f^{-1}(A)).$$
\end{definition}

\begin{definition}[Pullback metric]\label{D:pullback_metric}
Given a map $g:Y\rightarrow Z$ and a metric $d\in \mathcal{D}(Z)$, the pull back metric $g^\ast(d)\in\mathcal{D}(Y)$ is defined as: for any $y_1,y_2\in Y$,
$$(g^\ast d)(y_1,y_2):=d(g(y_1),g(y_2)).$$
\end{definition}

Manifold matching in our work refers to finding parameter $\theta_0$ of a specific generative network $f_\theta:\mathbb{R}^m\rightarrow\mathbb{R}^D$ such that the pushforward $(f_{\theta_0})_\ast\nu$ of some prior distribution $\nu$ via $f_{\theta_0}$ is condensed around a manifold $M\subseteq\mathbb{R}^D$.
In our case, the real data manifold $M\subseteq \mathbb{R}^D$ generally has no explicit expression. In other words, given a point $a\in\mathbb{R}^D$ these is no way to explicitly tell whether $a\in M$ or how far away $a$ is from $M$. For this reason, we attempt to estimate the shape of $M$ via a set of sample points from $M$.

The centroid of a space is an important descriptor of its shape. 
For a metric measure space, the Fréchet mean (cf.\,\cite{grove1973conjugatec,bhattacharya2003large,arnaudon2013medians}) is a natural generalization of the centroid: 
%In our work, we measure the shape of $M$ with respect to the metric $d$. In this situation, a proper choice of notion centroid whould be the the 
%of a the manifold $M$ can be estimated by the mean of a set of random sample points from $M$. 
\begin{definition}[Fréchet mean]\label{D:frechet}
%The Fréchet mean of a probability measure Q on a complete metric space(M, ρ)is the minimizer of the functionF(x)=∫ρ2(x, y)Q(dy),
The Fréchet mean set $\sigma(\mathcal{X})$ of a metric measure space $\mathcal{X}=(X,d,\mu)$ is defined as 
$$\underset{x\in X}{\operatorname{arg\,min}}\int_X d^2(x,y)d\mu(y).$$
\end{definition}

The Fréchet mean roughly informs the center of $\mathcal{X}$, but to reach the goal of manifold matching, we also need a shape descriptor indicating the size of $\mathcal{X}$. Hence we introduce the notion of $p$-diameter \cite{memoli2011gromov}: %offers a solution to approach the size of $M$ through random samples:
\begin{definition}[$p$-diameter]
For any $p\geq 1$, the $p$-diameter of metric measure space $\mathcal{X}:=(X,d,\mu)$ is defined as 
$$\operatorname{diam}_p(\mathcal{X}):=\left(\int_X\int_X d(x,x')^p d\mu(x)d\mu(x')\right)^{1/p}.$$
\end{definition}

The above definitions of Fréchet mean and $p$-diameter are for metric measure spaces, but it also applies to any manifold assuming a uniform volume measure on it. 
Let $S:=\{x_1,x_2,\cdots,x_k\}$ be a sequence of independent identically distributed points sampled from $\mu$. Let $\mu_k= \frac{1}{k}\Sigma_{i=1}^k\delta_{x_i}$ denote the empirical measure. We can estimate the shape of $(X,d,\mu)$ by the shape of $(S,d|_S,\mu_k)$. In the following we simply denote $\sigma(S):=\sigma(S,d|_S,\mu_k)$ and $\operatorname{diam}_p(S):=\operatorname{diam}_p(S,d|_S,\mu_k)$.

%In this paper, we consider the metric space $(X_n,d)$ where $X_n=\{x_1,\cdots,x_n\}$ is a set of identically independent random samples of a probability measure $\mu$.
%The $p$-th diameter can be used as important descriptor of the shape (or size) of a finite metric space. 
%we construct a finite metric measure space $\mathcal{X}_k:=(X_k, d|_{X_k}, \mu_k)$, 
%Particularly, a finite metric $(X_n,d)$ space $X_n=\{x_1,\cdots,x_n\}$ can be viewed as a metric measure space $(X_n,d,\mu_n)$, 
%where $\mu_n=\frac{1}{n}\Sigma_{i=1}^n\delta_{x_i}$ and $\delta_{x_i}$ represents the Dirac measure centered at $x_i$. 
%Then by Lemma~\ref{L:limits}, $\lim_{p\to\infty}\operatorname{diam}_\infty(X_n)=\max_{x,x'\in X}d(x,x')$. 
%; $\operatorname{diam}_1=\frac{1}{n^2}\Sigma_{i,j=1}^nd(x_i,x_j)$.

%\begin{figure}
%    \centering
%    \includegraphics[width=0.45\textwidth]{figs/p-diameter.pdf}
%    \caption{The $p$-diameter of a finite metric space as $p$ goes to $\infty$.}
%    \label{F:p-diam}
%\end{figure}

%As we can see from Fig.~\ref{F:p-diam}, when $p$ is relatively large, the $p$-diameter gets close to the diameter of 
When to estimate the $p$-diameter, if $p$ is relatively large, $\operatorname{diam}_p(S)$ is very sensitive to the outliers. To ensure a trustworthy estimation, we choose $p=2$ and use $\operatorname{diam}_2$ as a size indicator.

Let $S_R$ be a set of random real data samples of $M$ and $S_F$ be a set of random fake data samples of $\mu=(f_\theta)_\ast\nu$.
%with some prior probability measure $\nu$.
Then the objective function we propose for manifold matching is:
\begin{align}
\label{eqn:MM}
L_{MM} := & d\big(\sigma(S_R),\sigma(S_F)\big)+\nonumber\\
& \lambda\big|\operatorname{diam}_2(S_R)-\operatorname{diam}_2(S_F)\big|,
\end{align}
where $\lambda$ is a weight parameter.
%By Proposition~\ref{P:centroid_diff}, $L_{MM}$ can be easily calculated from the Euclidean metric $d_E$.
%\begin{align*}
%L_{MM}= \|\overline{g(x_R)} - \overline{g(x_F)}\|+\lambda \big|\operatorname{diam}_1(x_R)-\operatorname{diam}_1(x_F)\big|
%\end{align*} 

\subsection{Metric Learning}
%Many practical tasks rely on proper choice of metric on data sets. Especially in our work, 
Shape descriptors for manifold matching greatly rely on a proper choice of metric $d$. Although in most cases Euclidean metric $d_E$ is easy to access, it may not be an intrinsic choice and barely reveals the actual shape of a data set. 
The intrinsic metric on a Riemannian manifold $M$ is specified by geodesic distance. Specifically, the geodesic distance between two points on the manifold equals the length of shortest path on $M$ which connects them. From this point of view, a better choice of metric on the ambient space $\mathbb{R}^D\supseteq M$ should make the shortest path connecting $a,b\in M$ stay as close as possible to $M$, or in other words, make $M$ as ``straight'' as possible.
Here we apply Triplet metric learning to learn a proper metric on $\mathbb{R}^D$:
%Intuitively, a Triplet loss objective can help us ``straighten'' the real data manifold $M\subseteq\mathbb{R}^D$. (depicted in Fig.~\ref{F:distortion}).
%Triplet loss may compact real samples, however, adversarial learning mechanism tries to generate fake samples nearby the real data manifold. This behavior guarantees that the Triplet loss would not reduce its force to straighten the manifold along training process.
\begin{figure}[t]
    \centering
    \begin{tabular}{c c}
        \includegraphics[scale=0.11]{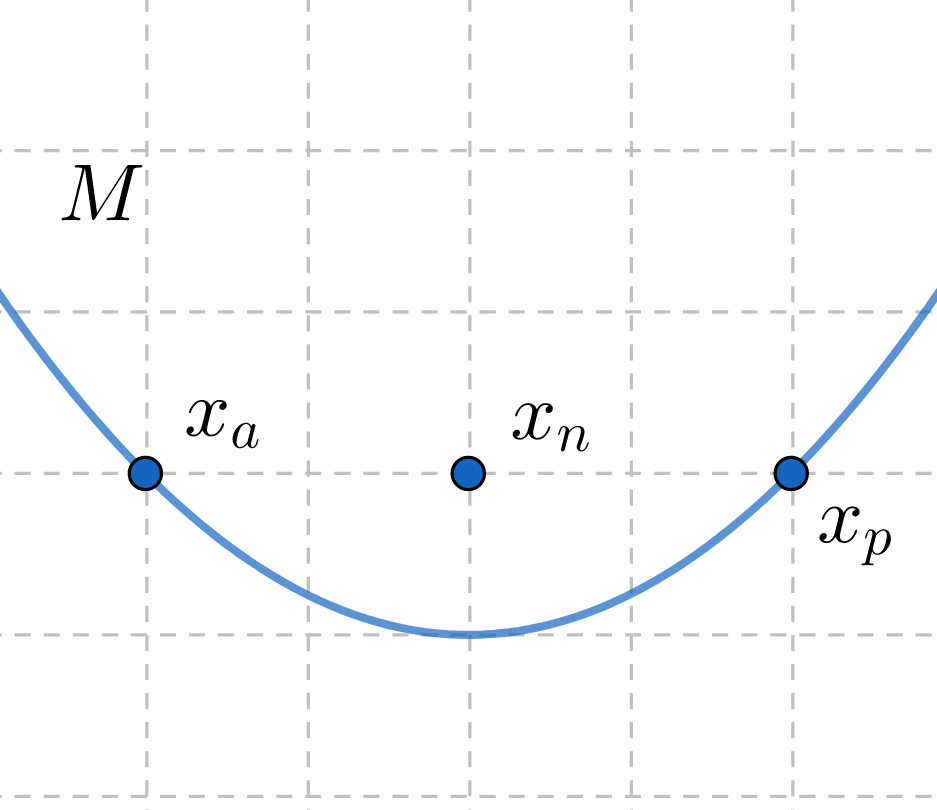}& \includegraphics[scale=0.11]{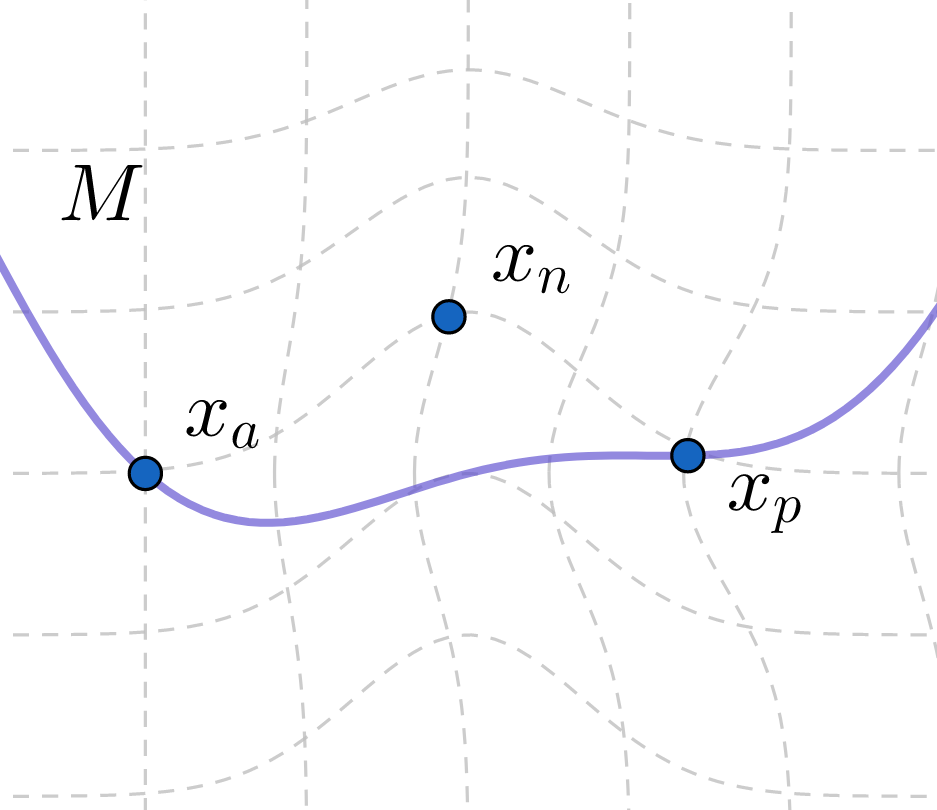}\\
        % (i) & (ii) 
    \end{tabular}
    \caption{Minimizing the Triplet loss pushes out negative sample $x_n$ and pulls back positive sample $x_p$. As a result, the learned metric would ``distort'' the space and ``straighten'' the manifold.
    }
    \label{F:distortion}
    \vspace{-1em}
\end{figure}
\begin{definition}
Given a triple $(x_a,x_p,x_n)$ with $x_a,x_p\in M$ and $x_n\notin M$, the Triplet loss is defined as
$$L_{tri}:=\max\{0, d^2(x_a,x_p)-d^2(x_a,x_n)+\alpha\}.$$
\end{definition}
Here $d=(g_w)^\ast d_E$ and $\alpha$ is a margin parameter.
People usually call $x_a$ an anchor sample, $x_p$ a positive sample and $x_n$ a negative sample.
\begin{figure*}[ht]
	\centering  
	\includegraphics[width=6.8in, height=1.8in]{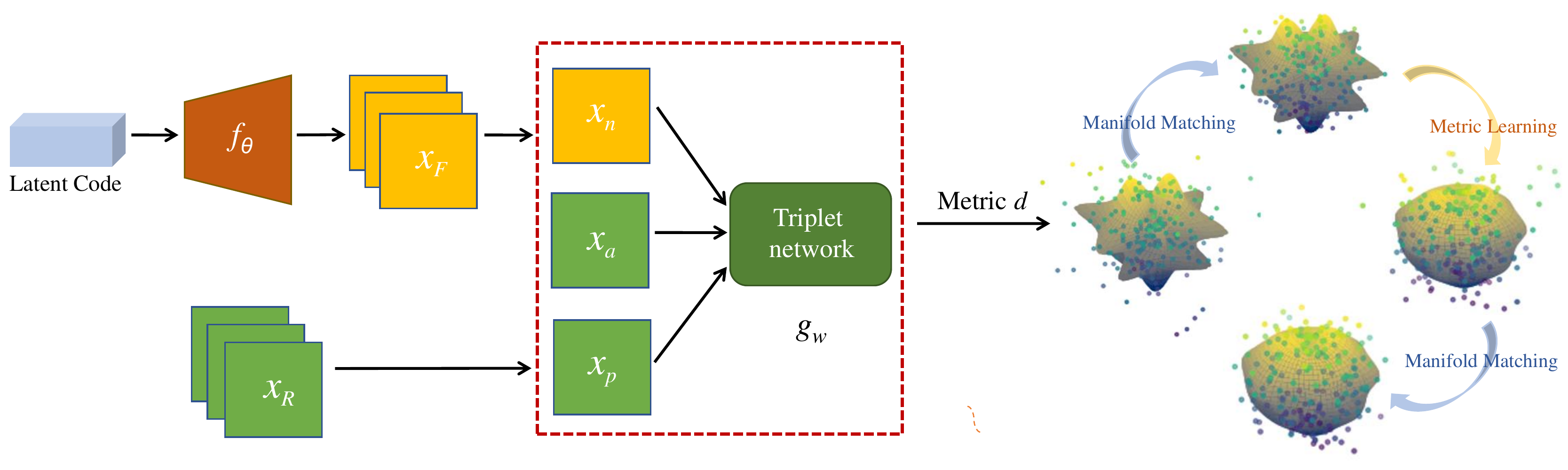}
	\caption{{\bf Implementation pipeline of the proposed MvM framework}. $x_R$ and $x_F$ represent samples in real data set $S_R$ and generated data set $S_F$, respectively.
	The distribution generator $f_\theta$ outputs samples $x_F \in S_F$ based on manifold matching criteria under learned metric. For Triplet metric learning, anchor samples $x_a$ and positive samples $x_p$ are randomly selected from $S_R$, and negative samples $x_n$ are randomly selected from $S_F$, without labelled data involved in this step. Learned distance metric is then used for manifold matching.
	Manifold matching step makes the fake samples (dots) condense around the real data manifold (surface), while metric learning step tries to ``straighen'' or ``flatten'' the real data manifold. These two steps goes interchangeably until convergence.} 
	\label{fig:network}
\end{figure*}
By minimizing $L_{tri}$, we attempt to pull back the positive sample to anchor and push out the negative sample, only when $d(x_a,x_p)$ is relatively larger than $d(x_a,x_n)$. Fig.~\ref{F:distortion} illustrates how this would ``straighten'' the manifold.
%% We may need to add this back later
%Note that the Triplet loss may compact the real samples while training, however, in our framework the adversarial learning mechanism always generates fake samples nearby the real data manifold. This behavior makes the triplet loss not reduce its force to straighten the real data manifold along the training process.

%It is observed that stochastic gradient descent converges poorly on optimizing the triplet loss. "
%In our work, we need the negative sample to move away 
%As we can see from Fig.~\ref{F:distortion}, to straighten real data manifold $M$, we need to move the negative sample $x_n$ away from $M$ instead of moving it forward to $M$. This is also a way to improve metric learning in other clustering tasks. 
%Several approaches are available to enforce this, such as using angular cost~\cite{wang2017deep}, using direction regularization~\cite{Mohan_2020_CVPR} et al.

Among numerous methods for metric learning, in our implementation we choose one recent approach~\cite{Mohan_2020_CVPR} which adapted Triplet loss by adding a direction regularizer to make the metric learned towards right direction. Hence in our paper the Triplet loss becomes:
\begin{align}
\label{eqn:apn}
L_{apn} &= \max\{0, d^2(x_a,x_p)-d^2(x_a,x_n) + \alpha - \nonumber\\ 
&\gamma Cos(g_w(x_n) - g_w(x_a), g_w(x_p) - g_w(x_a))\},
\end{align}
where $\gamma$ is the direction guidance parameter which controls the magnitude of regularization applied to the original Triplet loss $L_{tri}$. In practice one can also employ other methods to learn proper distance metrics.

%For our implementation, Triplet loss is applied to a batch size of random triples, in which anchor samples and positive samples are randomly chosen from $M$; negative samples are randomly chosen from fake data distribution $\mu$.

\subsection{Objective Functions}
% \begin{equation}
% \label{eqn:apn}
% \begin{split}
% L_{apn} = \|\Phi(I_R) - \Phi(I^{'}_R)\|^2 - \|\Phi(I_R) - \Phi(I_F)\|^2 + \\ \alpha - \gamma Cos(\Phi(I_F) - \Phi(I_R), \Phi(I^{'}_R) - \Phi(I_R))
% \end{split}
% \end{equation}

% \begin{equation}
% \label{eqn:MM}
% L_{MM} = \|\overline{\Phi(x_R)} - \overline{\Phi(x_F)}\| + \lambda_1 \|pd\{\Phi(x_R)\} - pd\{\Phi(x_F)\} \|
% \end{equation} 

% where $\overline{\Phi(I_R)}$ and $\overline{\Phi(I_F)}$ represent centroids of real and fake embedding respectively,  $pd\{\cdot\}$ represents the sum of pairwise distances between elements in $\cdot$, and $\lambda_1$ is a weight parameter.

 %By Proposition~\ref{P:centroid_diff}, 
%the Fréchet mean of $x_R$ and $x_F$ can be expressed as the mean vector of their Triplet embedding:

%where $\overline{g_w(x_R)}$ and $\overline{g_w(x_F)}$ represent mean vectors of real and fake Triplet network embedding, respectively. 
%Similarly, we use the pairwise $L^2$ distances between Triplet embedding of samples in a set to approximate the diameter of the set:
%$$\big|\operatorname{diam}_2(x_R)-\operatorname{diam}_2(x_F)\big| = \|pd\{g_w(x_R)\} - pd\{g_w(x_F)\} \|,$$
%where $pd\{\cdot\}$ represents the sum of pairwise distances between elements in $\cdot$.
In the metric learning community, the metric generator $g_w$ is usually viewed as a metric embedding. From this point of view we have (see supplement for the proof):
\begin{proposition}
Given two measure $\mu_1$ and $\mu_2$ on the same metric space $(X,d)$, where $d=g^\ast d_E$. Then $d\left(\sigma(X,d,\mu_1),\sigma(X,d,\mu_2)\right)=d_E(\overline{g_\ast\mu_1},\overline{g_\ast\mu_2})$.
\end{proposition}
Let $\| \cdot \|$ denote $L^2$ norm and $d=(g_w)^\ast d_E$, then we have $d(x,x') = \|g_w(x) - g_w(x')\|$ for $\forall x, x' \in \mathbb{R}^D$. By above proposition,
%Let $S_R$ be a set of sample points from real data manifold; let $x_F$ be a set of sample points from generated data distribution $\mu=(f_\theta)_\ast\nu$.
the explicit formula of terms in our objective functions (\ref{eqn:MM}) are as follows:
$$ d\big(\sigma(S_R),\sigma(S_F)\big) = \|\overline{g_w(S_R)} - \overline{g_w(S_F)}\|,$$
$$ \operatorname{diam}_2(S) =\frac{1}{card(S)}\left(\sum_{x,x'\in x_R}\|g_w(x)-g_w(x')\|^2\right)^{1/2}.$$
%\begin{align*}
% &Cos\left(g_w(x_n)-g_w(x_a),g_w(x_p)-g_w(x_a)\right)\\
%=&\frac{<g_w(x_n)-g_w(x_a),g_w(x_p)-g_w(x_a)>}{\|g_w(x_n)-g_w(x_a)\|\cdot\|g_w(x_p)-g_w(x_a)\|}.
%\end{align*}

For unconditional generation task, we take Eqn~(\ref{eqn:MM}) as manifold matching objective, and Eqn~(\ref{eqn:apn}) as metric learning objective.
During training we minimize both (\ref{eqn:MM}) and (\ref{eqn:apn}).
We display our implementation pipeline for unconditional image generation task in Fig.~\ref{fig:network} and summarize the training procedure in Algorithm $1$.
The convergence of training can be addressed using results from \cite{heusel2017gans}. Particularly, the setting in \cite{heusel2017gans} not only applies to min-max GANs, but is also valid for more general GANs where the discriminator’s objective is not necessarily related to the generator’s objective.  In our framework, using Adam optimizer with different decays for the two networks fits this setting.
\begin{algorithm} \label{algo:1}
	\caption{Metric learning assisted manifold matching}
	\hspace*{\algorithmicindent} \textbf{Input:} Real data manifold $M$, prior distribution $\nu$\\
	\hspace*{\algorithmicindent} \textbf{Output:} Distribution generator and metric generator parameters $\theta, w$
	\begin{algorithmic}  
		\WHILE {$\theta$ has not converged}
		\STATE {Sample real data set $S_R=\{x_1,\cdots,x_k\}$ from $M$;
		Sample random noise set $Z=\{z_1,\cdots,z_k\}$ from $\nu$;
		$\mathcal{L}_{MM} \leftarrow L_{MM}(S_R,f_\theta(Z))$  in which $d:=(g_w)^\ast d_E$;}
		\STATE $\theta \leftarrow$ Adam($\nabla_{\theta} \mathcal{L}_{MM}, \theta$)
		\STATE Sample $x^{(i)}_a,x^{(i)}_p$ from $M$ and sample $z^{(i)}$ from $\nu$, $i=1,2,\cdots,l$;
		\STATE $\mathcal{L}_{apn}\leftarrow\Sigma_{i=1}^l L_{apn}(x^{(i)}_a,x^{(i)}_p,f_\theta(z^{(i)}))$;
		\STATE $w \leftarrow$ Adam($\nabla_w \mathcal{L}_{apn}, w$);
		\ENDWHILE
	\end{algorithmic}
\end{algorithm}

As for super-resolution task, one can utilize information obtained from LR-HR pairs for more effective matching, thus we include an additional pair matching loss in this case:
$$
L_{pair} = \|g_w(x^{|}_R) - g_w(x^{|}_F) \|,
$$
where $x^{|}_F$ is super-resolved image from LR $x^{|}_L$, $x^{|}_L$ is downsampled image from HR $x^{|}_R$. 
%and $\lambda_2$ is a weight parameter.
We use $L^1$ norm for pixel-wise image loss, where
$L_{img} = \|x^{|}_R - x^{|}_F\|_1$.
Together, in SISR task our total data generator loss becomes
\begin{align}
\label{eqn:SR}
L_{gen} = L_{img} + \lambda_2 L_{pair} +  \lambda_3 L_{MM}.
\end{align}

\begin{figure}[h]
    \centering
    \includegraphics[width=0.48\textwidth,height=0.25\textwidth]{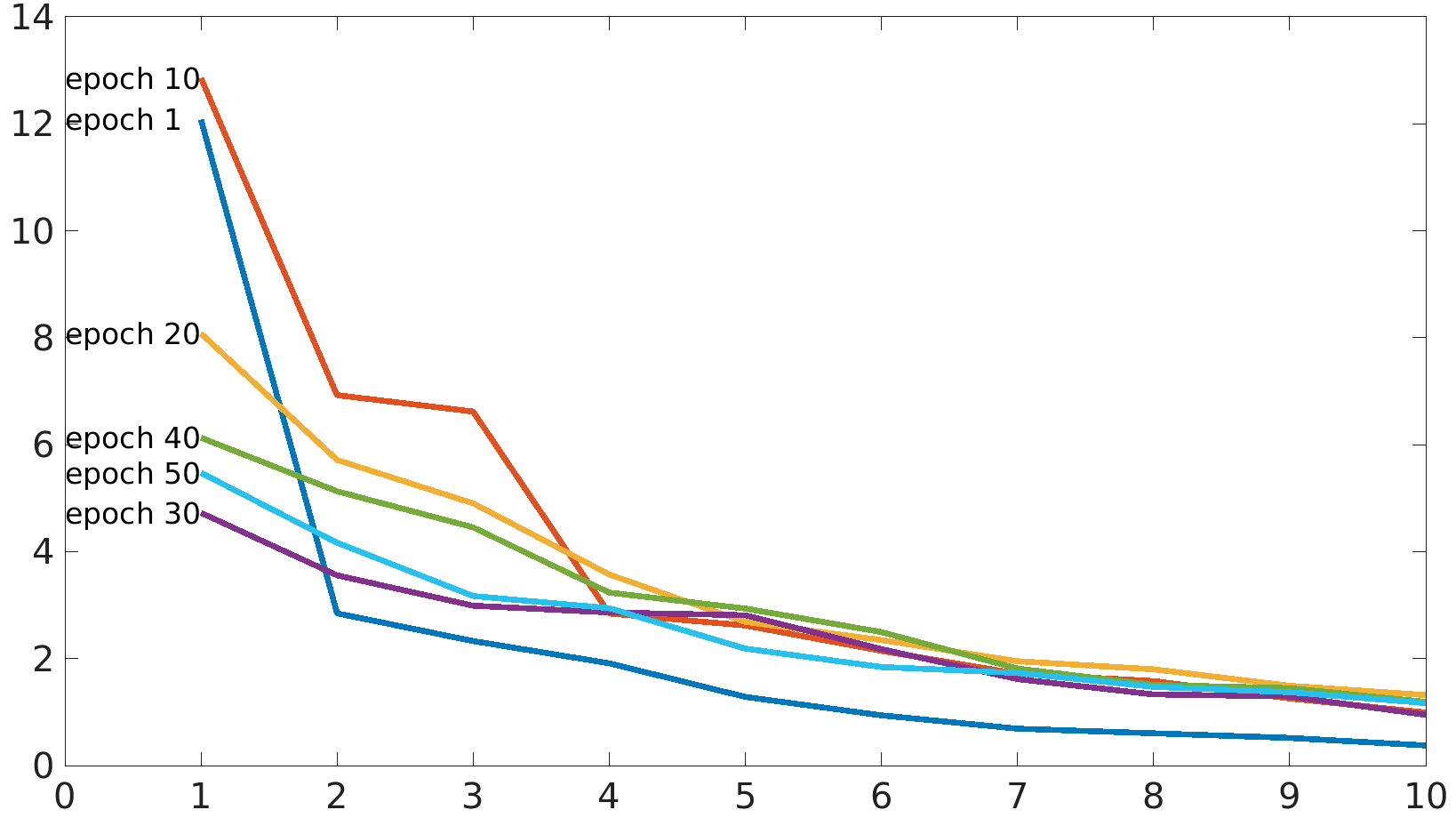}
    \caption{First 10 eigenvalues of distance matrices of 1024 random real data samples during training on CelebA.}
    \label{fig:eigens}
%        \vspace{-1em}
\end{figure}

\section{Experiments}
We conduct experiments on two tasks: unconditional image generation and single image super-resolution. 
To better understand the process of manifold matching, in both tasks we track various distances in batches during training. We represent distance between centroids of two sets as $d_c$, distance between $2$-diameters of two sets as $d_g$, and distance between paired SR-HR samples in super-resolution task as $d_p$, respectively. We also use Hausdorff distance as a measurement of distance between the two sets. Hausdorff distance is defined as ${\displaystyle d_{\mathrm {H} }(A,B)=\max \left\{\,\sup _{a\in A}\inf _{b\in B}d(a,b),\,\sup _{b\in B}\inf _{a\in A}d(a,b)\,\right\},\!}$, where $A$ and $B$ are two non-empty subsets of a metric space ($\mathcal{M}, d$). Here we adopt Euclidean metric to calculate Hausdorff distance between two sets of Triplet network embedding, which is equivalent to measure Hausdorff distance between two sets of corresponding images in the image space under learned metric. 
All experiments are implemented under PyTorch framework using a Tesla V100 GPU.

\subsection{Unconditional Image Generation}
%Training of SISR models utilizes LR and HR paired information, and the effects of matching geometric descriptors between point sets may not be easily revealed. 
We use the matching criteria in Eqn~\ref{eqn:MM} to validate the feasibility of the proposed framework. Note that no paired information or GAN loss is used for the task.

\noindent 
{\bf Implementation Details:} 
We employed a ResNet data generator and a deep convolutional net metric generator with $\gamma=0.01$, $\lambda=1$, dimension of input latent vector $m=128$, and output embedding $n = 10$ as our default setting. Adam optimizer with learning rate $1e-4$, $\beta _1 = 0$ and $\beta _2 = 0.9$ for data generator, and $\beta _1 = 0.5$ and $\beta _2 = 0.999$ for metric generator was used during training. Details of network architectures are provided in supplementary material.

\noindent 
{\bf Dataset and Evaluation Metrics:} We implemented our method on CelebA \cite{liu2015faceattributes} and LSUN bedroom \cite{journals/corr/YuZSSX15} 
%and CIFAR-10 \cite{Krizhevsky09learningmultiple} 
datasets. For training we used around 200K images in CelebA and 3M images in LSUN. All images were center-cropped and resized to $32 \times 32$ or $64 \times 64$.  For each dataset we randomly generated 50K samples and used Fréchet Inception Distance (FID) \cite{heusel2017gans} for quantitative evaluation. Smaller FID indicates better result.
%MNIST \cite{lecun-mnisthandwrittendigit-2010}

%    \vspace{-1em}
\noindent
{\bf Effect of Metric Learning on Distorting Real Data Manifold:}
%In order to address how the learned metric gradually distort the real data manifold, 
During training we randomly choose 1024 real samples and detect its shape by looking at the eigenvalues of the corresponding (normalized) distance matrix. Particularly, we plot the first 10 largest eigenvalues which correspond to the size of the first 10 principle components of real samples. As shown in Fig.~\ref{fig:eigens}, the shape of the samples becomes more uniform with training going. This is an empirical evidence that $g_w$ distorts real data manifold to be uniformly curved.
In this situation, matching centers and diameters should be enough for manifold matching.\\
\noindent 
{\bf Effects of Matching Different Geometric Descriptors:}
%Fig~\ref{fig:tsne_mnist}(c) displays a typical failed case where the centroids of two ``circles'' are close but shapes of them are very different. 
%Fig~\ref{fig:tsne_mnist}(d) shows another scenario when the two point sets have similar $2$-diameters, but their centoirds are not located nearby.
%In experiments we found the data generator was able to produce some visually meaningful results in case (c) and failed in case (b). 
We study the effect of matching different geometric descriptors.
Examples of generated $32 \times 32$ samples on CelebA using different matching criteria are shown in Fig.~\ref{fig:celeba_3in1}. (i) Centroid matching learns some common shallow patterns from real data set; (ii) Matching $2$-diameters could capture more complicated intrinsic structures of the data manifold, while misalignment between the two sets can result in low-quality samples (e.g. image on the right side in the second last row); (iii) Combining the two descriptors together leads to more stable sample quality.   
We also visualize these manifold matching status for illustration purpose. We project output of $g_w$ to $2$-dim plane using UMAP \cite{McInnes2018}, and display the projected points in Fig.~\ref{fig:celeba_3in1} (a)(b)(c). 
%At the beginning stage (a) two data sets are separated clearly. After training for some time (d), they largely overlap with each other in the projected space. 
%which indicates the corresponding sub-manifolds are ``close'' to each other 
%(or on the contrary if we fail to learn meaningful metrics in the first place). 
% The former case would lead to a successful training session.
(a) and (b) intuitively show two typical matching status if one only matches centroids or $2$-diameters between the two sets, respectively. \\
\begin{figure}[ht]
    \centering
    \begin{tabular}{ccc}
    {\includegraphics[width=.15\textwidth]{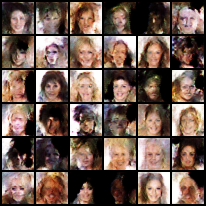}}\hspace{-1em} & 
    {\includegraphics[width=.15\textwidth]{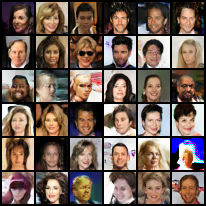}}\hspace{-1em} &  
    {\includegraphics[width=.15\textwidth]{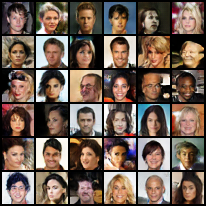}}
    \\
    (i)\hspace{-1em} & (ii) \hspace{-1em} & (iii)\\
    {\includegraphics[width=.15\textwidth]{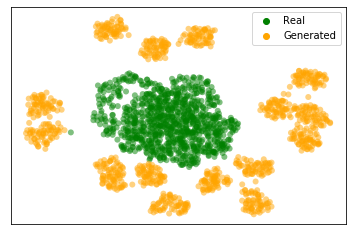}}\hspace{-1em} & 
    {\includegraphics[width=.15\textwidth]{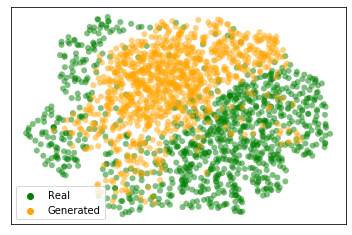}}\hspace{-1em} &  
    {\includegraphics[width=.15\textwidth]{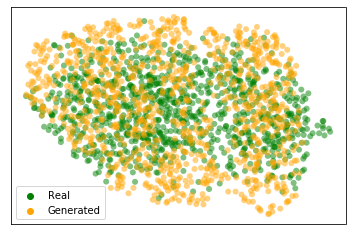}}
    \\
    (a)\hspace{-1em} & (b) \hspace{-1em} & (c)\\
    \end{tabular}
    \caption{Randomly generated $32 \times 32$ images on CelebA with different manifold matching criteria during training. (i) centroid only; (ii) $2$-diameter only; (iii) both centroid and $2$-diameter. (a)(b)(c) are corresponding UMAP plots of real (green) and generated (orange) samples with different manifold matching criteria. (a) centroid only; (b) $2$-diameter only; (c) both.}
    \label{fig:celeba_3in1}
    %\vspace{-1em}
\end{figure}
%We track the UMAP plots for all three cases during training and store them as movies attached in supplementary material.
% \begin{figure}[ht]
%     \centering
%     \begin{tabular}{cc}
%     {\includegraphics[width=.21\textwidth]{figs/mnist_epoch1.png}}\hspace{-1em} \vspace{-0.25em} & 
%     {\includegraphics[width=.21\textwidth]{figs/mnist_epoch50_dc.png}} \\
%     (a)\hspace{-1em} \vspace{-0.25em}& (b) \\
%     {\includegraphics[width=.21\textwidth]{figs/mnist_epoch50_dg.png}}\hspace{-1em} \vspace{-0.25em}& 
%     {\includegraphics[width=.21\textwidth]{figs/mnist_epoch50.png}} \\
%     (c)\hspace{-1em} \vspace{-0.25em}& (d) \\
%     \end{tabular}
%     \caption{UMAP plot of real (green) and generated (orange) samples with different manifold matching criteria.(a) starting epoch; (b) centroid only; (c) $2$-diameter only; (d) both.}
%     \label{fig:tsne_mnist}
%     \vspace{-1em}
% \end{figure}
\begin{figure}[h]
	\centering  
	\includegraphics[width=3.2in]{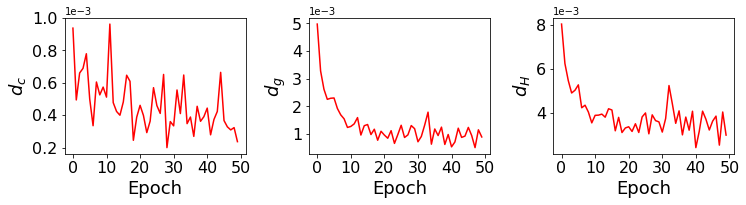}  
	\caption{Distances versus training epochs in unconditional image generation task.} 
	\label{fig:dist_mnist}
	%\vspace{-0.5em}
\end{figure}
\begin{table}[h]
    \small % \scriptsize
	\centering
	\caption{Comparisons of results trained on different batch sizes.}
	\label{tab:batch_size}
	\begin{tabular}{c|c|c|c|c|c|c}
		\hline
		Batch size & 32 & 64 & 128 & 256 & 512 & 1024 \\
		\hline
		Effective training & \cmark  & \cmark & \cmark & \cmark & \cmark & \xmark\\
		\hline
		Time(s) / epoch & 178 & 155 & 139 & 127 & 127 & - \\
		\hline
	\end{tabular}
    \vspace{-1em}
\end{table}

\noindent 
{\bf Effects of Batch Size:} 
We further study the influence of batch size on training time and stability. Table~\ref{tab:batch_size} reports average training time per epoch on resized $32 \times 32$ CelebA images. One can see batch size does not have a significant influence on average training time when matching $2$-diameters. In addition, we observe stable and efficient training sessions with different batch sizes. For batch size as large as $1024$, we did not observe satisfying sample quality in reasonable training time. 

\begin{table*}[ht]
%    \small % \scriptsize
	\centering
	\caption{Evaluation scores of different training settings in $\times 4$ SISR task. Demonstration of settings is displayed in Table~\ref{tab:setup}. For each pair of settings with the same generator backbone, the one with better performance is highlighted.} 
	\label{tab:SR_GAN_MM}
	\setlength{\tabcolsep}{3pt}
	\begin{tabular}{l|cccc|cccc|cccc|cccc}
	    \hline
	    \toprule
        &  \multicolumn{4}{c}{Set5} & \multicolumn{4}{c}{Set14} & \multicolumn{4}{c}{BSD100} & \multicolumn{4}{c}{Urban100}\\
        \midrule
		\hline
		 Setting
		 &\scriptsize PSNR &\scriptsize SSIM &\scriptsize LPIPS &\scriptsize NIQE
		 &\scriptsize PSNR &\scriptsize SSIM &\scriptsize LPIPS &\scriptsize NIQE
		 &\scriptsize PSNR &\scriptsize SSIM &\scriptsize LPIPS &\scriptsize NIQE
		 &\scriptsize PSNR &\scriptsize SSIM &\scriptsize LPIPS &\scriptsize NIQE\\
		\hline
	    ResNet-GAN &\scriptsize 29.03 &\scriptsize 0.8468 &\scriptsize \bf 0.1885  &\scriptsize 7.2143 
	    &\scriptsize 25.64 &\scriptsize 0.7420 &\scriptsize 0.2761  &\scriptsize \bf 5.2654 
	    &\scriptsize 25.74 &\scriptsize 0.7026 &\scriptsize \bf 0.3160 &\scriptsize \bf 5.3896 
	    &\scriptsize 23.49 &\scriptsize 0.7273 &\scriptsize 0.2888 &\scriptsize \bf 4.6504\\
	    ResNet-MvM &\scriptsize \bf 29.76 &\scriptsize \bf 0.8606 &\scriptsize 0.1941  &\scriptsize \bf 6.1478 
	    &\scriptsize \bf 26.12 &\scriptsize \bf 0.7562 &\scriptsize \bf 0.2724  &\scriptsize 5.3458 
	    &\scriptsize \bf 26.07 &\scriptsize \bf 0.7150 &\scriptsize 0.3178  &\scriptsize 5.5508
	    &\scriptsize \bf 24.06 &\scriptsize \bf 0.7505 &\scriptsize \bf 0.2774 &\scriptsize 4.7549\\
	    \hline
	    RDN-GAN &\scriptsize 29.07 &\scriptsize 0.8442 &\scriptsize \bf 0.1841 &\scriptsize 6.6459
	    &\scriptsize 25.38 &\scriptsize 0.7355 &\scriptsize 0.2729 &\scriptsize 5.3887
	    &\scriptsize 25.72 &\scriptsize 0.7029 &\scriptsize \bf 0.3063 &\scriptsize 5.6734
	    &\scriptsize 23.11 &\scriptsize 0.7190 &\scriptsize 0.2876 &\scriptsize 4.6546\\
	    RDN-MvM &\scriptsize \bf 30.06 &\scriptsize \bf 0.8658 &\scriptsize 0.1850  &\scriptsize \bf 6.1283
	    &\scriptsize \bf 26.31 &\scriptsize \bf 0.7615 &\scriptsize \bf 0.2641 &\scriptsize \bf 5.2161
	    &\scriptsize \bf 26.24 &\scriptsize \bf 0.7210 &\scriptsize 0.3100 &\scriptsize \bf 5.4127
	    &\scriptsize \bf 24.44 &\scriptsize \bf 0.7645 &\scriptsize \bf 0.2566 &\scriptsize \bf 4.5914\\
	    \hline
	    NSRNet-GAN &\scriptsize 29.46 &\scriptsize 0.8544 &\scriptsize 0.1852 &\scriptsize \bf 5.7818 
	    &\scriptsize 25.93 &\scriptsize 0.7478 &\scriptsize 0.2666 &\scriptsize \bf 5.1213
	    &\scriptsize 25.93 &\scriptsize 0.7094 &\scriptsize 0.3119 &\scriptsize \bf 5.2069 
	    &\scriptsize 23.70 &\scriptsize 0.7330 &\scriptsize 0.2832 &\scriptsize 5.1579\\
	    \hline
	    NSRNet-MvM &\scriptsize \bf 29.79 &\scriptsize \bf 0.8641 &\scriptsize \bf 0.1845 &\scriptsize 6.0846
	    &\scriptsize \bf 26.17 &\scriptsize \bf 0.7590 &\scriptsize \bf 0.2655 &\scriptsize 5.3175
	    &\scriptsize \bf 26.13 &\scriptsize \bf 0.7188 &\scriptsize \bf 0.3114 &\scriptsize 5.3794 
	    &\scriptsize \bf 24.19 &\scriptsize \bf 0.7569 &\scriptsize \bf 0.2621 &\scriptsize \bf 4.5937\\
% 	    C & {\scriptsize  / /6.3969}& & &\\
% 	    D &{\scriptsize 29.30/0.8445/6.6590}&{\scriptsize 25.81/0.7394/5.5318} &{\scriptsize 25.81/0.7010/5.8399} & {\scriptsize 23.83/0.7346/4.6377}\\
% 	    E & {\scriptsize 29.70/0.8603/6.2673}&{\scriptsize 26.06/0.7556/ {\bf 5.3777}} &{\scriptsize 26.05/0.7146/5.6095} &{\scriptsize 23.98/0.7477/4.7880}\\
% 		\hline
% 		F (ours) &{\scriptsize {\bf 30.02/0.8654/6.1876}} &{\scriptsize {\bf 26.26/0.7611}/5.3986} & {\scriptsize {\bf 26.21/0.7197/5.5723}} &{\scriptsize {\bf 24.36/0.7619/4.7439}}\\
		\hline
	\end{tabular}
    %\vspace{-1em}
\end{table*}

\noindent 
{\bf Quantitative Results:}
We track $d_c, d_g$ and $d_H$ during training and display them in Fig.~\ref{fig:dist_mnist}. 
With training going forward, the distances keep decreasing and gradually converge. The observation aligns with our manifold matching assumption even with no labelled information involved.
For quantitative evaluation we present FID scores in Table~\ref{tab:FID_image64}. Here we also display results from some classic GAN frameworks using the same generator architecture. 
%For WGAN \cite{wgan2017} and MMD GAN \cite{li2017mmd} we report available resultsin \cite{Ansari_2020_CVPR} . 
As shown in the table our method obtains competitive results. 
\begin{table}[h]
    \small
	\centering
	\caption{FID evaluation on $64 \times 64$ experiments with a ResNet generator.}
	\label{tab:FID_image64}
	\begin{tabular}{c|c|c}
		\hline
		Method & CelebA & LSUN bedroom \\
		\hline
		WGAN \cite{wgan2017}    & 37.1 (1.9)  & 73.3 (2.5) \\
		WGAN-GP \cite{NIPS2017_892c3b1c} & 18.0 (0.7)  & 26.9 (1.1) \\
		SNGAN \cite{miyato2018spectral}  & 21.7 (1.5)  & 31.3 (2.1) \\
		SWGAN \cite{Wu_2019_CVPR}    & 13.2 (0.7)  & 14.9 (1.0) \\
		\hline
		MvM  & {\bf 11.1 (0.1)} & {\bf 13.7 (0.3)} \\
		\hline
	\end{tabular}
	%\vspace{-1em}
\end{table}
%with the other two approaches, meanwhile indeed it is more efficient. Our training session is $3$x faster compared to WGAN-GP to obtain the displayed results.  
% Fig~\ref{fig:sample_3in1} displays examples of randomly generated samples on four datasets. As we see the proposed approach successfully generates recognizable samples in various classes. 
% \begin{figure}[h]
% 	\centering  
% 	\includegraphics[width=3.2in, height=2.3in]{figs/samples32.png} 
% 	\caption{Randomly generated $32 \times 32$ samples on four datasets: CelebA, LSUN bedroom, CIFAR-10 and MNIST.} 
% 	\label{fig:sample_3in1}
% 	\vspace{-1em}
% \end{figure}
% \begin{table}[h]
% 	\centering
% 	\caption{FID evaluation on $32 \times 32$ experiments with a DCGAN generator.}
% 	\label{tab:FID}
% 	\begin{tabular}{cccc}
% 		\hline
% 		Method & CelebA & LSUN & CIFAR-10 \\
% 		\hline
% 		WGAN-GP  & 10.03 (0.37) & 23.85 (0.22) & 35.91 (0.30)\\
% 		MMD-rq  & 13.22 (1.30) & 26.43 (1.04)  & 38.88 (1.35)\\
% 		\hline
% 		Ours  & {\bf 6.37 (0.10)} & {\bf 20.62 (0.28)} & {\bf 25.21 (0.37)}\\
% 		\hline
% 	\end{tabular}
% 	\vspace{-1em}
% \end{table}
% MNIST 1.25 (0.04)
%More examples are presented in supplementary material.
%To study the effectiveness of the proposed framework on other image sizes, we further take CelebA dataset as an example and experimented on resized $64 \times 64$ images. 
Examples of randomly generated samples are displayed in Fig.~\ref{fig:sample64}. %As we see the proposed approach generated visually appealing samples on the two datasets.
\begin{figure}[ht]
	\centering  
	\includegraphics[width=3.2in, height=3in]{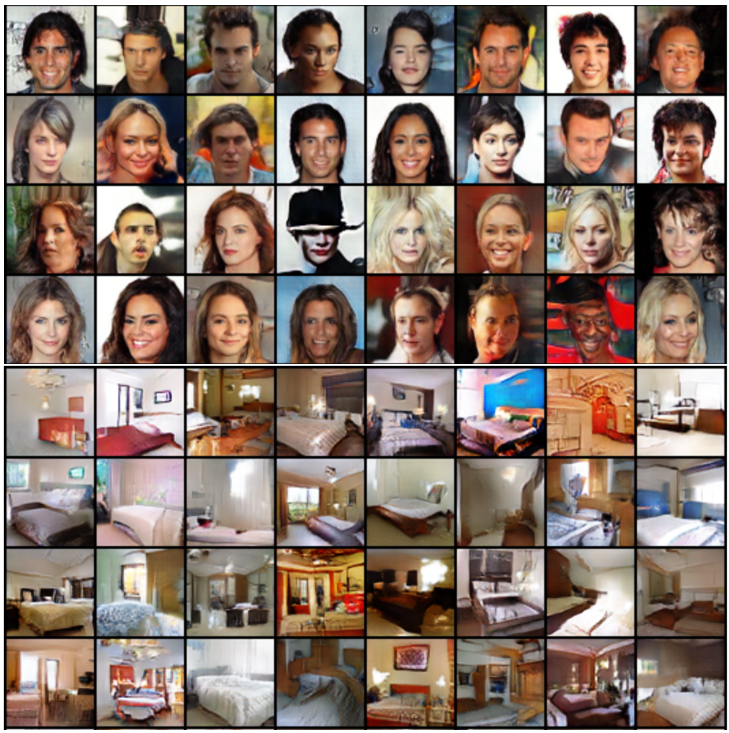} 
	\caption{Randomly generated $64 \times 64$ samples on CelebA and LSUN bedroom.} 
	\label{fig:sample64}
	 %\vspace{-1em}
\end{figure}

% \begin{table}[ht]
%     \small
% 	\centering
% 	\caption{FID evaluation on $64 \times 64$ CelebA experiments with a ResNet generator.}
% 	\label{tab:FID_celeba64}
% 	\begin{tabular}{l|c}
% 		\hline
% 		Method   & FID $(\downarrow)$ \\
% 		\hline
% 		WGAN-GP  & 18.0(0.7) \\
% 		SNGAN   & 21.7(1.5) \\
% 		CTGAN  & 15.8(0.6) \\
% 		SWGAN  & 13.2(0.7) \\
% 		\hline
% 		Ours   & {\bf 11.1(0.2)} \\
% 		\hline
% 	\end{tabular}
% 	\vspace{-1em}
% \end{table}

\subsection{Single Image Super-Resolution}
In a typical perception-based SR model \cite{ledig2017photo,Soh_2019_CVPR}, the generator loss function is usually made up by three components: pixel-wise image loss, GAN loss and perceptual loss (or naturalness loss in \cite{Soh_2019_CVPR}).
Here we explore the use of our work with two different settings: (A) Our approach (MvM) serves as a substitute of GAN loss. In this case we use Eqn~\ref{eqn:SR} as the total generator loss function without involving GAN loss; (B) MvM serves as a substitute of naturalness loss. In (B) MvM is used as a complement of GAN in perception-based models, where a GAN loss is added to Eqn~\ref{eqn:SR} as the total generator loss. 

%There are various factors that could influence results, including training data, network architecture, objective to optimize, optimization method and so on. Thus it is hard to compare different methods in that which factors result in the gains.
\noindent 
{\bf Implementation Details:} 
For setting (A), we compare MvM and GAN using three different generator backbones: ResNet \cite{ledig2017photo}, RDN \cite{zhang2018residual} and NSRNet \cite{Soh_2019_CVPR}. All experiments were conducted with the same training setup. 
For setting (B), we compare the effects of different perceptual components in perception-based models. We employed the ResNet architecture in \cite{ledig2017photo} as the default generator backbone, and RaGAN \cite{DBLP:journals/corr/abs-1807-00734} with weight $2e-3$ as GAN component for consistency. 
For both (A) and (B) we experimented on $\times 4$ SR task. We utilized Adam optimizer with learning rate $1e-4$, $\beta_1 = 0.9$, $\beta_2 = 0.999$ , metric learning direction guidance parameter $\gamma=1e-2$, weight of diameter matching term $\lambda=1$,  dimension of Triplet network output embedding $n = 32$, batch size $=32$ , $\lambda_2  = \lambda_3  = 1e-3$, and trained for 100K iterations. Details of Triplet network architecture is presented in supplementary material.

\begin{figure}[ht]
	\centering  
	\includegraphics[width=3.3in]{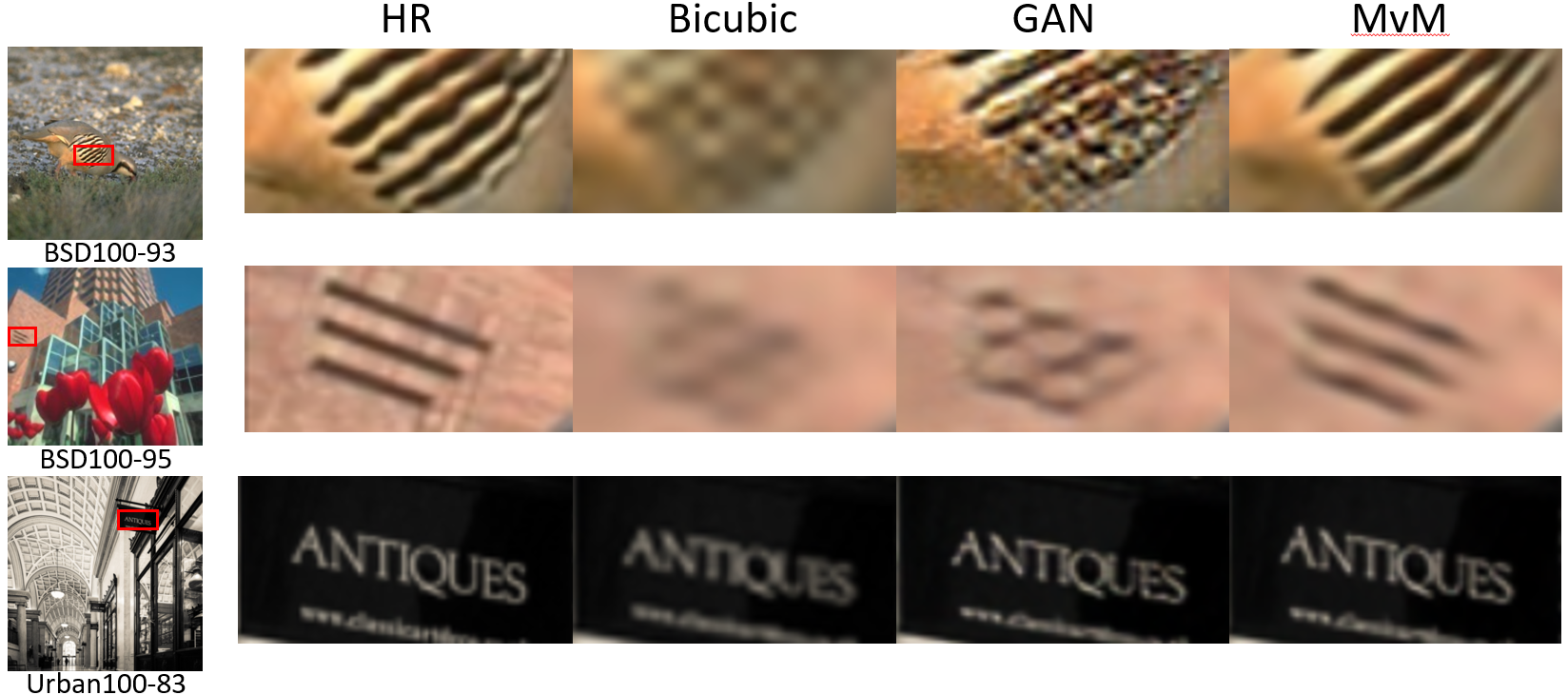}  
	\caption{Generated $\times 4$ samples with GAN loss and MvM loss using the same generator backbone. (Zoom in for better view.)} 
	\label{fig:SR1}
    %\vspace{-1em}
\end{figure} %width=6in, height=3.5in

\begin{table*}[ht]
%    \small % \scriptsize
	\centering
	\caption{Evaluation scores of different training settings in $\times 4$ SISR task. Demonstration of settings is displayed in Table~\ref{tab:setup}. The best performance is highlighted in red and second best in blue.} 
	\label{tab:SSIM_PSNR}
	\setlength{\tabcolsep}{3pt}
	\begin{tabular}{c|cccc|cccc|cccc|cccc}
	    \hline
	    \toprule
        &  \multicolumn{4}{c}{Set5} & \multicolumn{4}{c}{Set14} & \multicolumn{4}{c}{BSD100} & \multicolumn{4}{c}{Urban100}\\
        \midrule
		\hline
		 Setting
		 &\scriptsize PSNR &\scriptsize SSIM &\scriptsize LPIPS &\scriptsize NIQE
		 &\scriptsize PSNR &\scriptsize SSIM &\scriptsize LPIPS &\scriptsize NIQE
		 &\scriptsize PSNR &\scriptsize SSIM &\scriptsize LPIPS &\scriptsize NIQE
		 &\scriptsize PSNR &\scriptsize SSIM &\scriptsize LPIPS &\scriptsize NIQE\\
		\hline
	    A &\scriptsize 28.42 &\scriptsize 0.8104 &\scriptsize 0.2977  &\scriptsize 7.3647 
	    &\scriptsize 26.00 &\scriptsize 0.7027 &\scriptsize 0.3757  &\scriptsize 6.3393 
	    &\scriptsize 25.96 &\scriptsize 0.6675 &\scriptsize 0.4637 &\scriptsize 6.2967 
	    &\scriptsize 23.14 &\scriptsize 0.6577 &\scriptsize 0.4182 &\scriptsize 6.6429\\
	    \hline
	    B &\scriptsize 29.66 &\scriptsize 0.8586 &\scriptsize 0.1401  &\scriptsize 6.6582 
	    &\scriptsize 26.08 &\scriptsize 0.7542 &\scriptsize 0.2265  &\scriptsize 5.3638 
	    &\scriptsize 26.05 &\scriptsize 0.7145 &\scriptsize 0.3257  &\scriptsize 5.5376 &
	    \scriptsize 24.00 &\scriptsize 0.7464 &\scriptsize 0.2156 &\scriptsize 4.7883\\
	    \hline
	    C &\scriptsize 29.30 &\scriptsize 0.8445 &\scriptsize 0.1324 &\scriptsize 6.6590 
	    &\scriptsize 25.81 &\scriptsize 0.7394 &\scriptsize 0.2132 &\scriptsize 5.5318
	    &\scriptsize 25.81 &\scriptsize 0.7010 &\scriptsize 0.3006 &\scriptsize 5.8399 
	    &\scriptsize 23.83 &\scriptsize 0.7346 &\scriptsize 0.2047 &\scriptsize 4.6377\\
	    \hline
	    D &\scriptsize 29.75 &\scriptsize 0.8593 &\scriptsize 0.1357  &\scriptsize 6.5475
	    &\scriptsize \textcolor{blue}{26.11} &\scriptsize 0.7542 &\scriptsize 0.2193 &\scriptsize 5.3559
	    &\scriptsize \textcolor{blue}{26.06} &\scriptsize 0.7145 &\scriptsize 0.3160 &\scriptsize 5.6108
	    &\scriptsize \textcolor{blue}{24.15} &\scriptsize \textcolor{blue}{0.7506} &\scriptsize 0.2028 &\scriptsize 4.7447\\
	    \hline
	    E &\scriptsize 29.63 &\scriptsize 0.8547 &\scriptsize \textcolor{red}{0.1227} &\scriptsize 6.2838 
	    &\scriptsize 26.07 &\scriptsize 0.7512 &\scriptsize \textcolor{blue}{0.2051} &\scriptsize \textcolor{blue}{5.0657} 
	    &\scriptsize 26.05 &\scriptsize 0.7119 &\scriptsize \textcolor{blue}{0.2957} &\scriptsize \textcolor{blue}{5.2257} 
	    &\scriptsize 24.12 &\scriptsize 0.7484 &\scriptsize \textcolor{red}{0.1862} &\scriptsize \textcolor{red}{4.2814}\\
	    \hline
	    F &\scriptsize \textcolor{blue}{29.70} &\scriptsize \textcolor{blue}{0.8603} &\scriptsize 0.1409 &\scriptsize \textcolor{blue}{6.2673}
	    &\scriptsize 26.06 &\scriptsize \textcolor{blue}{0.7556} &\scriptsize 0.2239 &\scriptsize 5.3777
	    &\scriptsize 26.05 &\scriptsize \textcolor{red}{0.7146} &\scriptsize 0.3218 &\scriptsize 5.6095 
	    &\scriptsize 23.98 &\scriptsize 0.7477 &\scriptsize 0.2126 &\scriptsize 4.7880\\
	    \hline
	    G &\scriptsize \textcolor{red}{29.87} &\scriptsize \textcolor{red}{0.8615} &\scriptsize \textcolor{blue}{0.1281} &\scriptsize \textcolor{red}{5.8904} 
	    &\scriptsize \textcolor{red}{26.17} &\scriptsize \textcolor{red}{0.7564} &\scriptsize \textcolor{red}{0.2048} &\scriptsize \textcolor{red}{4.9612}
	    &\scriptsize \textcolor{red}{26.10} &\scriptsize \textcolor{blue}{0.7134} &\scriptsize \textcolor{red}{0.2930} &\scriptsize \textcolor{red}{4.9864} 
	    &\scriptsize \textcolor{red}{24.25} &\scriptsize \textcolor{red}{0.7567} &\scriptsize \textcolor{blue}{0.1866} &\scriptsize \textcolor{blue}{4.3432}\\ %4.3432
% 	    C & {\scriptsize  / /6.3969}& & &\\
% 	    D &{\scriptsize 29.30/0.8445/6.6590}&{\scriptsize 25.81/0.7394/5.5318} &{\scriptsize 25.81/0.7010/5.8399} & {\scriptsize 23.83/0.7346/4.6377}\\
% 	    E & {\scriptsize 29.70/0.8603/6.2673}&{\scriptsize 26.06/0.7556/ {\bf 5.3777}} &{\scriptsize 26.05/0.7146/5.6095} &{\scriptsize 23.98/0.7477/4.7880}\\
% 		\hline
% 		F (ours) &{\scriptsize {\bf 30.02/0.8654/6.1876}} &{\scriptsize {\bf 26.26/0.7611}/5.3986} & {\scriptsize {\bf 26.21/0.7197/5.5723}} &{\scriptsize {\bf 24.36/0.7619/4.7439}}\\
		\hline
	\end{tabular}
    \vspace{-1em}
\end{table*}

\noindent 
{\bf Datasets and Evaluation Metrics:} We used 800 HR images in DIV2K \cite{Agustsson_2017_CVPR_Workshops} dataset as training set, and four benchmark datasets: Set5 \cite{BMVC.26.135}, Set14 \cite{Zeyde12onsingle}, BSD100 \cite{martin2001database} and Urban100 \cite{Huang_2015_CVPR} for testing. HR images were downsampled by bicubic interpolation to get $48 \times 48$ LR input patches. We evaluated results using Structure Similarity (SSIM) \cite{Wang04imagequality}, PSNR as distortion-oriented evaluation metrics, and Naturalness Image Quality Evaluator (NIQE) \cite{mittal2012making},  Learned Perceptual Image
Patch Similarity (LPIPS) \cite{zhang2018unreasonable} as perception-oriented evaluation metrics. Lower NIQE and LPIPS scores indicate higher perceptual quality. 

\begin{table}[ht]
    \scriptsize %\small
	\centering
	\caption{Training settings for GAN-based SISR methods with fixed generator architecture and GAN component.}
	\label{tab:setup}
	\begin{tabular}{c|c|c|c}
		\hline
		Setting & Method & Perceptual component & Pretrained \\
		\hline
		A & Bicubic & - & -\\
		\hline
		B & GAN-ResNet & - & -\\
		\hline
		C & SRGAN & VGG16-Pooling5 & Yes\\
		\hline
		D & EnhanceNet & VGG19-Pooling2,5 & Yes\\
		\hline
		E & ESRGAN & VGG19-Conv5-4 & Yes\\
		\hline
		F & NatSR & NMD & Yes\\
		\hline
		G(ours) & MvM & MM & No\\
		\hline
	\end{tabular}
    \vspace{-1em}
\end{table}

% \noindent 
% {\bf Results:} 

\noindent
{\bf Results:}
We record $d_p, d_c, d_g$, and Hausdorff distance $d_H$ during training and display them in Fig.\ref{fig:sr_dist}. One can see all distances keeps decreasing with training going forward, which is aligned with our basic assumption that the two sets of points keep getting close to each other. 
Note that the spike effect is common for Adam optimizer when feeding some abnormal random batch of training data. 
\begin{figure}[ht]
	\centering  
	\includegraphics[width=3.2in]{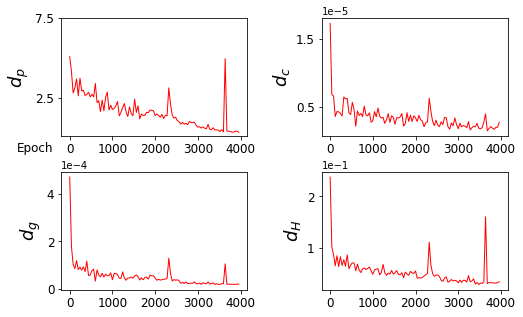}  
	\caption{Various distances versus training epochs in super-resolution experiment with our method. Distances are displayed every 40 epochs (1000 iterations).} 
	\label{fig:sr_dist}
    \vspace{-1em}
\end{figure}

\noindent 
{\bf (A) MvM As A Substitute of GAN Loss:} 
We display comparison of evaluation results between GAN and MvM in Table~\ref{tab:SR_GAN_MM}. Under different backbones, MvM performs better than GAN in similarity-based metrics in all cases. It also obtains better scores in perception-based metrics in majority of the cases.  
Examples of generated samples using the same generator backbone are displayed in Fig.~\ref{fig:SR1}. We see MvM resulted in samples with more natural textures. 

\noindent 
{\bf (B) MvM As A Substitute of Naturalness Loss:}
We display a few different setups in Table~\ref{tab:setup}.
Final results with both distortion-based and perception-based evaluation scores on benchmark datasets are presented in Table~\ref{tab:SSIM_PSNR}, where GAN-ResNet represents SRGAN without VGG component in loss function. With the same common setup, MvM obtained better results in most of the cases for both types of metrics.
Examples of generated samples from various settings are shown in Fig.~\ref{fig:SR_comp}. As we see MvM generates samples with more natural textures and less artifacts under the same setup.
\begin{figure}[ht]
	\centering  
	\includegraphics[width=3.3in]{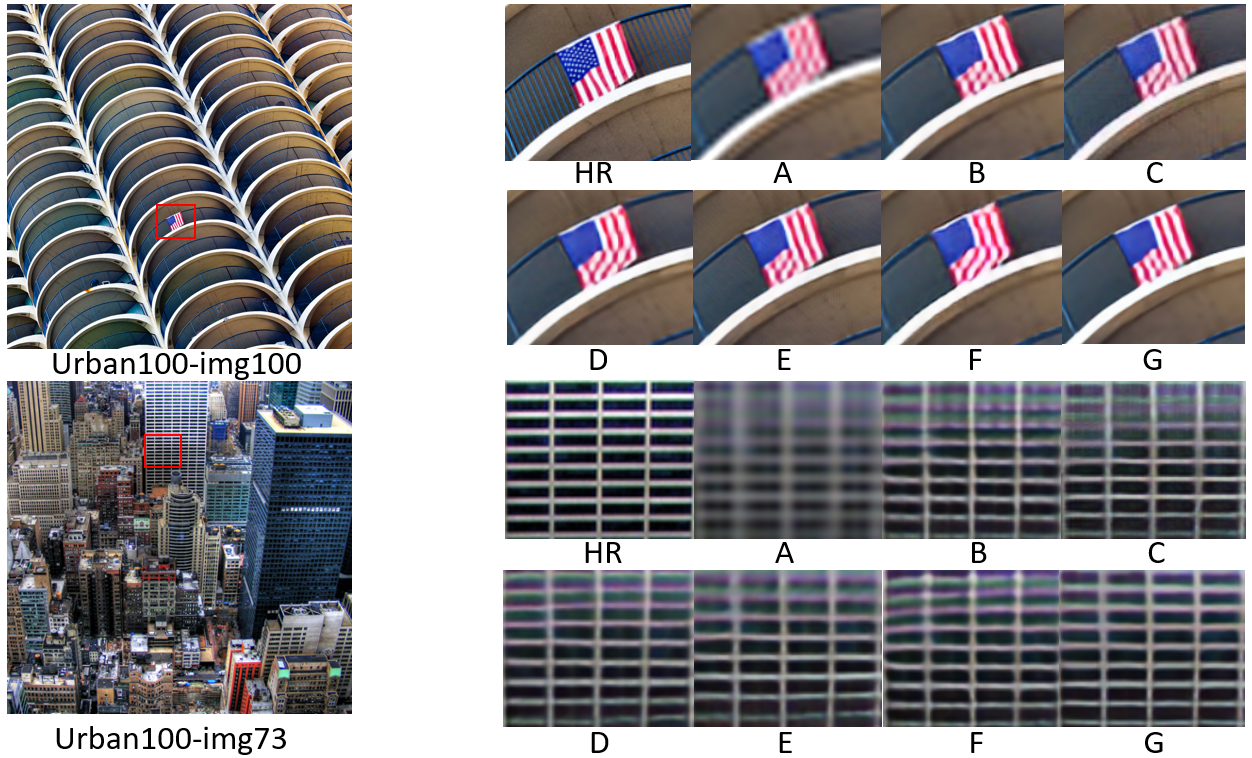}  
	\caption{Comparison of generated samples from different settings. All settings were trained with the same ResNet generator backbone and GAN components. (Zoom in for better view.)} 
	\label{fig:SR_comp}
    \vspace{-2em}	
\end{figure} %width=6in, height=3.5in
We notice that although both GAN and MvM utilize adversarial learning for training, the two approaches behave differently in super-resolution task. GAN tends to generate more details but with artifacts, while MvM tends to generate more natural textures in images. The two approaches do not conflict with each other. Instead, one serves as a complement for the other to result in better sample quality.   

% In Fig~\ref{fig:comp_BSD100_95} we display examples involving two other types of generator networks \cite{zhang2018residual,Soh_2019_CVPR} for comparison. Not only our method generates more realistic visual results with the same generator backbone (E,F,G), it also cooperates well with other generator architectures (G-RDN and G-NSR). To obtain better visual qualities, one could always utilize more powerful generator architectures \cite{Guo_2020_CVPR,Mei_2020_CVPR,Liu_2020_CVPR}, add data augmentation stacks \cite{Yoo_2020_CVPR}, train longer sessions or tune parameters. To control the degree of sharpness versus artifacts, one can also consider adjusting weights of distortion-oriented and perception-based generator components as suggested in \cite{wang2018esrgan}.   

% \begin{figure}[ht]
% 	\centering  
% 	\includegraphics[width=3.2in]{figs/comp_BSD100_95.png}  
% 	\caption{Comparison of generated samples from different settings. G-RDN and G-NSR represent Setting G with generator architecture in RDN \cite{zhang2018residual} and NSRNet \cite{Soh_2019_CVPR} respectively.} 
% 	\label{fig:comp_BSD100_95}
% \end{figure}

\section{Discussion and Conclusion}
In this paper, we have proposed a manifold matching approach for generative modeling, which matches geometric descriptors of real and generated data sets using learned distance metrics.
Experiments on two tasks validated its feasibility and effectiveness.
Moreover, the proposed framework is robust and flexible in that each network has its own designated objective.
Despite that our method has led to some promising results, there is yet much room for improvements.
For example, the currently used geometric descriptors may not fully recover the information of the underlying manifolds, thus matching towards other descriptors could potentially benefit the learning process. 
As to metric learning, in our paper we employ the method in~\cite{Mohan_2020_CVPR} fed with random samples to learn a metric, while in practice one could investigate ways for better approximation of metrics, such as utilizing sampling methods or other metric learning methods.
In addition, although empirical evidences and intuitions agree with the current metric measure space setting, further theoretical analysis using optimal transport is worth exploring. 
In the last few years we have witnessed great success of probability-based generative modeling approaches, and we believe joining geometry with statistics would lead to stronger expression ability for generative models, which is a promising direction for researchers to explore. 

\section{Acknowledgment}
MD and HH would like to thank the four anonymous reviewers especially Reviewer 1 and Reviewer 4 for their valuable comments. HH acknowledges partial support by NSF grant DMS-1854683.

\clearpage

\bibliographystyle{ieee_fullname}%
\bibliography{mmml}

\clearpage

\section{Proofs of Properties}
Given any space $X$ associated with a probability measure $\mu$ s.t. $\operatorname{supp}[\mu]=X$ and $p\geq 1$, then for any function $f:X\rightarrow\mathbb{R}$ we denote its $L^p$ norm as 
$$\|f\|_p:=\left(\int_X |f(x)|^pd\mu(x)\right)^{1/p}.$$
Then we have the following properties:
\begin{lemma}\label{L:lp_norm} 
Denote $\|f\|_\infty:=\sup_{x\in X}f(x)$, then
\begin{enumerate}
    \item If $p\leq q$, then $\|f\|_p\leq\|f\|_q$;
    \item $\lim_{p\to \infty}\|f\|_p=\|f\|_{\infty}$.
\end{enumerate}
\end{lemma}
\begin{proof}
Let $r=q/p$ and $s$ be such that $1/r+1/s=1$. Let $a(x)=|f(x)|^p$ and $b(x)\equiv 1$, then
by Holder's inequality, we have 
$$\|ab\|_1\leq\|a\|_r\|b\|_s$$ or 
$$\int_{x\in X} |f(x)|^p d\mu(x)\leq \left(\int_{x\in X}|f(x)|^{p\cdot \frac{q}{p}}d\mu(x)\right)^{p/q}
.$$ 
Now taking the $p$-th root on both sides we obtain $\|f\|_p\leq \|f\|_q$.

It is easy to see $\|f\|_p\leq \|f\|_\infty$ for any $p>1$. Now we choose $\forall \epsilon>0$ with $\epsilon<\|f\|_\infty$ and let $X_\epsilon:=\{x\in X|f(x)\geq \|f\|_\infty-\epsilon\}$, then
\begin{align*}
& \left(\int_{x\in X}|f(x)|^p d\mu(x)\right)^{1/p} \\
\geq & \left(\int_{X_\epsilon}(\|f\|_\infty-\epsilon)^p d\mu(x)\right)^{1/p} \\
= & (\|f\|_\infty-\epsilon) \mu(X_\epsilon)^{1/p}
\end{align*}
When $p\to \infty$ we have $\lim_{p\to\infty}\|f\|_p\geq \|f\|_\infty-\epsilon$. Since $\epsilon$ is arbitrary we have $\lim_{p\to\infty}\|f\|_p\geq \|f\|_\infty$.
\end{proof}

\section{Properties of Geometric Descriptors}
Given metric measure space $\mathcal{X}=(X,d,\mu)$, intuitively, the Fréchet mean set $\sigma(\mathcal{X})$ represents the center of $\mathcal{X}$ with respect to metric $d$. Particularly,
\begin{lemma}\label{L:mean}
The Fréchet mean of measure $\mu$ with respect to $d_E$ coincides with the mean $\overline{\mu}:=\int_{\mathbb{R}^D}yd\mu(y)$.
\end{lemma}
\begin{proof}
Suppose that elements in $\mathbb{R}^D$ are represented as column vectors. Then $\sigma(\mathbb{R}^D,d,\mu)$ equals the set of minimizers of $$F(x)=\int_{\mathbb{R}^D}(y-x)^T(y-x) d\mu(y).$$ Since $\frac{\partial F}{\partial x}=-2\int_{\mathbb{R}^D}yd\mu(y)+2x$, let $\frac{\partial F}{\partial x}=0$ we have single minimizer of $F(x)$ to be $\int_{\mathbb{R}^D}yd\mu(y)=\overline{\mu}$.
\end{proof}
When the underlying metric $d$ is not the Euclidean metric, the Fréchet mean provides a better option of centroid.
%\begin{proposition}
%Given a metric measure space $\mathcal{X}:=(X,d,\mu)$ and map $g:X\rightarrow\mathbb{R}^n$. If $d=g^\ast d_E$, then 
%$$\sigma(\mathcal{X})=g^{-1}(\overline{g_\ast\mu}).$$
%\end{proposition}

Let $x_1,x_2,\cdots,x_k,\cdots$ be a sequence of independent identically distributed points sampled from $\mu$. Let $\mu_k= \frac{1}{k}\Sigma_{i=1}^k\delta_{x_i}$ denote the empirical measure. We can estimate the Fréchet mean of $(X,d,\mu)$ by a set of random samples:
\begin{proposition}
$\sigma(\mathcal{X})=\lim_{k\to\infty}\sigma(X,d,\mu_k).$
\end{proposition}
\begin{proof}
Since the empirical measure $\mu_k$ weakly converges to $\mu$, the result following by the definition of weak convergence.
\end{proof}

%%%%%%%%%%%%%%%%%%%%%%%%%%%%%%%%%%%%%%%%%%%%%%%%%%%%%%%%%%%%

The following lemma explains its geometric meaning of $p$-diameters:
\begin{lemma}\label{L:limits}
For any metric measure space $\mathcal{X}:=(X,d,\mu)$ with $\operatorname{supp}[\mu]=X$,
\begin{enumerate}
\item $\operatorname{diam}_p(\mathcal{X})\leq \operatorname{diam}_q(\mathcal{X})$ for any $p\leq q$;
\item $\lim_{p\to \infty}\operatorname{diam}_p(\mathcal{X})=\sup\limits_{x,x'\in X}d(x,x')$.
\end{enumerate}
\end{lemma}
\begin{proof}
It follows directly from Lemma~\ref{L:lp_norm}
\end{proof}
Similarly, we can estimate the $p$-diameter of a metric measure space $\mathcal{X}:=(X,d,\mu)$ by a set of random samples:
\begin{proposition}
$$\operatorname{diam}_p(\mathcal{X})=\lim_{k\to\infty}\operatorname{diam}_p(X,d,\mu_k).$$
\end{proposition}
\begin{proof}
Since the empirical measure $\mu_k$ weakly converges to $\mu$, the result following by the definition of weak convergence.
\end{proof}
%%%%%%%%%%%%%%%%%%%%%%%%%%%%%%%%%%%%%%%%%%%%%%%%%%%%%%%%%%%%%

\begin{proposition}\label{P:embedding}
Given a metric measure space $\mathcal{X}:=(X,d,\mu)$ and map $g:X\rightarrow\mathbb{R}^n$. If $d=g^\ast d_E$, then 
$$\sigma(\mathcal{X})=g^{-1}(\overline{g_\ast\mu}).$$
\end{proposition}
\begin{proof}
By definition, $\sigma(\mathcal{X})$ is the set of minimizers of function 
\begin{align*}
F(x) :&= \int_X (g^\ast d_E)^2(x,y)d\mu(y)\\
&=\int_X d^2_E(g(x),g(y))d\mu(y)\\
&=\int_{\mathbb{R}^n} d^2_E(g(x),z) d(g_\ast\mu)(z).
\end{align*}
It is easy to see that $x_0$ is a minimizer of $F(x)$ iff. $g(x_0)$ is a minimizer of $G(w):=\int_{\mathbb{R}^n}d^2_E(w,z)d(g_\ast\mu)(z)$. By Lemma 3.4, $\overline{g_\ast\mu}=\operatorname{argmin}G(w)$, hence we have $\sigma(\mathcal{X})=\operatorname{argmin} F(x)=g^{-1}(\overline{g_\ast\mu})$.
\end{proof}

\begin{proof}[Proof of Proposition 3.10]
It follows directly from Proposition~\ref{P:embedding}
\end{proof}

\section{Network Architectures}
For unconditional image generation task, we used a ResNet data generator and a deep convolutional net metric generator:

\noindent
$f_\theta$: convt(128) $\rightarrow$ upres(128) $\rightarrow$ upres(128) $\rightarrow$ upres(128) $\rightarrow$ upres(128) $\rightarrow$ bn $\rightarrow$ conv(128) $\rightarrow$ sig;

\noindent 
$g_w$: conv(32) $\rightarrow$ leaky-relu(0.2) $\rightarrow$ conv(64) $\rightarrow$ leaky-relu(0.2) $\rightarrow$ conv(128) $\rightarrow$ leaky-relu(0.2) $\rightarrow$ conv(256) $\rightarrow$ leaky-relu(0.2) $\rightarrow$ conv(512) $\rightarrow$ leaky-relu(0.2) $\rightarrow$ maxpool $\rightarrow$ dense(10).

For super-resolution task, the architecture of Triplet embedding network is presented as below:

\noindent
$g_w$:  (conv(x$\rightarrow$2x, $x_0$=32) $\rightarrow$ prelu $\rightarrow$ maxpool) * 7  $\rightarrow$ dense(256) $\rightarrow$ prelu $\rightarrow$ dense(256) $\rightarrow$ prelu $\rightarrow$ dense(32),

\noindent
where conv, convt, upres, bn, relu, leaky-relu, prelu, maxpool, dense and sig refer to nn.Conv2d, nn.ConvTranspose2d, up-ResidualBlock, nn.BatchNorm2d, nn.ReLU,
nn.LeakyReLU, nn.PReLu, nn.MaxPool2d, nn.Linear and nn.Sigmoid layers in Pytorch framework respectively.

\section{More Evaluation Details}
\noindent
{\bf Perception-Based Evaluations in SISR:} we adopted the function $niqe$ in Matlab for computing NIQE scores, and python package $lpips$ for computing LPIPS.

\noindent
{\bf FID Evaluation:} we used 50000 randomly generated samples comparing against 50000 random samples from real data sets for testing. Features were extracted from the pool3 layer of a pre-trained Inception network. FID was computed over 10 bootstrap resamplings. \\

\end{document}